%% file: emnlp2020.tex
\documentclass[11pt,a4paper]{article}
\usepackage[hyperref]{emnlp2020}
\usepackage{times}
\usepackage{latexsym}

\usepackage{amsmath}
\usepackage{amssymb}
\usepackage{url}
\usepackage{graphicx}
\usepackage{booktabs}
\usepackage{amsthm}
\usepackage{xspace}
\usepackage{mleftright}
\usepackage{multirow}
\usepackage{xcolor}
\usepackage[normalem]{ulem}
\usepackage{mathtools}
\definecolor{reddish}{HTML}{FF4136}
\definecolor{bluish}{HTML}{0074D9}
\usepackage{thmtools}

\usepackage{tikz-dependency}

\newcommand{\dyckkm}{Dyck-($k$,$m$)\xspace}

\newcommand{\symend}{\omega}
\newcommand{\symopen}{{\color{bluish} \langle}}
\newcommand{\symclose}{{\color{bluish} \rangle}}
\newcommand{\rscale}{\lambda}
\newcommand{\pscale}{\zeta}

\newtheorem{lemma}{Lemma}
\numberwithin{equation}{section}

\theoremstyle{definition}
\newtheorem{defn}{Definition}

\usepackage{microtype}

\aclfinalcopy %
\def\aclpaperid{3550} %

\title{%
RNNs can generate bounded hierarchical languages\\ with optimal memory}
\author{John Hewitt\textsuperscript{$\dagger$} \ \ \ Michael Hahn\textsuperscript{$\ddagger$} \ \ \ Surya Ganguli\textsuperscript{$\|$} \ \ \  
 \textbf{Percy Liang\textsuperscript{$\dagger$} \ \ \ Christopher D. Manning\textsuperscript{$\dagger$}\textsuperscript{$\ddagger$}}\\
 \textsuperscript{$\dagger$}Computer Science Department \ \ \ \textsuperscript{$\ddagger$}Linguistics Department \ \ \ \textsuperscript{$\|$}Applied Physics Department\\
  Stanford University \\
  \{johnhew,mhahn2,sganguli,pliang,manning\}@stanford.edu} 

\date{}

\begin{document}
\maketitle
\begin{abstract}
  Recurrent neural networks empirically generate natural language with high syntactic fidelity.
  However, their success is not well-understood theoretically.
  We provide theoretical insight into this success, proving in a finite-precision setting that RNNs can efficiently generate bounded hierarchical languages that reflect the scaffolding of natural language syntax.
  We introduce \dyckkm, the language of well-nested brackets (of $k$ types) and $m$-bounded nesting depth, reflecting the bounded memory needs and long-distance dependencies of natural language syntax.
  The best known results use $O(k^{\frac{m}{2}})$ memory (hidden units) to generate these languages.
  We prove that an RNN with $O(m\log k)$ hidden units suffices, an exponential reduction in memory, by an explicit construction. %
  Finally, we show that no algorithm, even with unbounded computation, can suffice with $o(m\log k)$ hidden units.
\end{abstract}

\section{Introduction}
Recurrent neural networks (RNNs; \citet{elman1990finding}) trained on large datasets have demonstrated a grasp of natural language syntax \cite{karpathy2015visualizing,kuncoro2018lstms}.
While considerable empirical work has studied RNN language models' ability to capture syntactic properties of language \cite{linzen2016assessing,marvin2018targeted,hewitt2019structural,vanschijndel2018gardenpath}, their success is not well-understood theoretically. %
In this work, we provide theoretical insight into RNNs' syntactic success, proving that they can efficiently generate a family of bounded hierarchical languages.
These languages form the scaffolding of natural language syntax.

\begin{figure}
  \includegraphics[width=\linewidth]{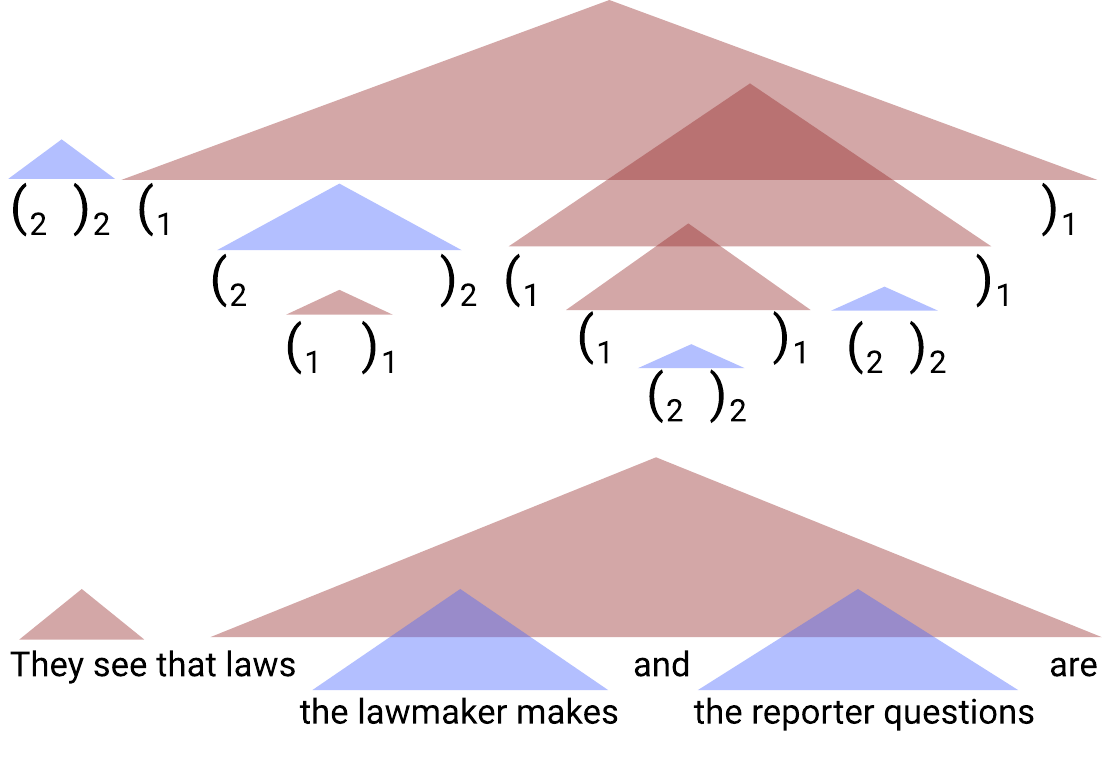}
  \caption{\label{figure_head} (Top) This string of well-nested brackets is a member of the Dyck-($2$,$4$) language; triangles denote the scopes of nested hierarchical dependencies, mirroring the core of hierarchical structure in natural languages. (Bottom) A fragment in English with similar nested dependencies denoted by triangles.
  }
\end{figure}

Hierarchical structure characterized by long distance, nested dependencies, lies at the foundation of natural language syntax.
This motivates, for example, context-free languages \cite{chomsky1956three}, a fundamental paradigm for describing natural language syntax.
A canonical family of  context-free languages (CFLs) is Dyck-$k$, the language of balanced brackets of $k$ types, since any CFL can be constructed via some Dyck-$k$ \cite{chomsky1959algebraic}. 

However,  while context-free languages like Dyck-$k$ describe arbitrarily deep nesting of hierarchical structure, in practice, natural languages exhibit bounded nesting.
This is clear in, e.g., bounded center-embedding \cite{karlsson2007constraints,jin2018depth} (Figure~\ref{figure_head}).
To reflect this, we introduce and study \dyckkm, which adds a bound $m$ on the number of unclosed open brackets at any time.
Informally, the ability to efficiently generate \dyckkm suggests the foundation of the ability to efficiently generate languages with the syntactic properties of natural language.
See \S(\ref{sec_motivating_bounded_hierarchical_structure}) for further motivation for bounded hierarchical structure.

In our main contribution, we prove that RNNs are able to generate \dyckkm as memory-efficiently as any model, up to constant factors.
Since \dyckkm is a regular (finite-state) language, the application of general-purpose RNN constructions trivially proves that RNNs can generate the language  \cite{merrill2019sequential}.
However, the best construction we are aware of uses $O(k^{\frac{m}{2}})$ hidden units \cite{horne1994bounds,indyk1995optimal}, where $k$ is the vocabulary size and $m$ is the nesting depth, which is exponential.\footnote{For hard-threshold neural networks, the lower-bound would be $\Omega(k^{\frac{m}{2}})$ if \dyckkm were an arbitrary regular language.}
We provide an explicit construction proving that a Simple (Elman; \citet{elman1990finding}) RNN can generate any \dyckkm using only $6m\lceil\log k\rceil-2m=O(m\log k)$ hidden units, an exponential improvement.
This is not just a strong result relative to RNNs' general capacity; we prove that even computationally unbounded models generating \dyckkm require $\Omega(m \log k)$ hidden units, via a simple communication complexity argument.

Our proofs provide two explicit constructions, one for the Simple RNN and one for the LSTM, which detail how these networks can use their hidden states to simulate stacks in order to efficiently generate \dyckkm.
The differences between the constructions exemplify how LSTMs can use \textit{exclusively} their gates, ignoring their Simple RNN subcomponent entirely, to reduce the memory by a factor of 2 compared to the Simple RNN.\footnote{We provide implementations at \url{https://github.com/john-hewitt/dyckkm-constructions/}.}

We prove these results under a theoretical setting that aims to reflect the realistic settings in which RNNs have excelled in NLP\@.
First, we assume \textit{finite precision}; the value of each hidden unit is represented by $p=O(1)$ bits.
This has drastic implications compared to existing work \cite{siegelmann1992computational,weiss2018practical,merrill2019sequential,merrill2020formal}.
It implies that \textbf{only regular languages} can be generated by any machine with $d$ hidden units, since they can take on only $2^{pd}$ states \cite{korsky2019computational}.
This points us to focus on whether languages can be implemented memory-efficiently.
Second, we consider networks as language \textit{generators}, not acceptors;\footnote{Acceptors consume a whole string and then decide whether the string is in the language; generators must decide which tokens are possible continuations at each timestep.} informally, RNNs' practical successes have been primarily as generators, like in language modeling and machine translation \cite{karpathy2015visualizing,wu2016google}. %

Finally, we include a preliminary study in learning \dyckkm with LSTM LMs from finite samples, finding for a range of $k$ and $m$ that learned LSTM LMs extrapolate well given the hidden sizes predicted by our theory.

In summary, we prove that RNNs are memory optimal in generating a family of bounded hierarchical languages that forms the scaffolding of natural language syntax by describing mechanisms that allow them to do so; this provides theoretical insight into their empirical success.

\subsection{Motivating bounded hierarchical structure} \label{sec_motivating_bounded_hierarchical_structure}
Hierarchical structure is central to human language production and comprehension, showing up in grammatical constraints and semantic composition, among other properties \cite{chomsky1956three}.
Agreement between subject and verb in English is an intuitive example:
\begin{center}
  \small 
\begin{dependency}[hide label, edge unit distance=.4ex]
\begin{deptext}[column sep=0.05cm]
  Laws \& the lawmaker \& the reporter \& questions \& writes \& are \& (1) \\
\end{deptext}
\depedge{1}{6}{.}
\depedge{2}{5}{.}
\depedge{3}{4}{.}
\end{dependency}
\end{center}
The Dyck-$k$ languages---well-nested brackets of $k$ types---are the prototypical languages of hierarchical structure; by the Chomsky-Sch{\"u}tzenberger Theorem \cite{chomsky1959algebraic}, they form the scaffolding for any context-free language.
They have a simple structure: %
\begin{center}
\begin{dependency}[hide label, edge unit distance=.4ex]
\begin{deptext}[column sep=0.04cm]
    $\symopen_1$\& $\symclose_1$\&$\symopen_2$\& $\symopen_2$\&$\symopen_1$\&$\symclose_1$\&$\symclose_2$\&$\symopen_1$\&  $\symclose_1$\&$\symclose_2$\\
\end{deptext}
\depedge{1}{2}{.}
\depedge{3}{10}{.}
\depedge{4}{7}{.}
\depedge{5}{6}{.}
\depedge{8}{9}{.}
\end{dependency}
\end{center}
However, human languages are unlike Dyck-$k$ and other context-free languages in that they exhibit \textit{bounded memory requirements}.
Dyck-$k$ requires storage of an unboundedly long list of open brackets in memory.
In human language, as the center-embedding depth grows, comprehension becomes more difficult, like in our example sentence above \cite{miller1963finitary}.
Empirically, center-embedding depth of natural language is rarely greater than $3$ \cite{jin2018depth,karlsson2007constraints}.
However, it does exhibit long-distance, shallow hierarchical structure:
\begin{center}
  \small 
\begin{dependency}[hide label, edge unit distance=.4ex]
\begin{deptext}[column sep=0.04cm]
  Laws \& the lawmaker \& wrote \& along with the motion ... \& are \\
\end{deptext}
\depedge{1}{5}{.}
\depedge{2}{3}{.}
\end{dependency}
\end{center}
Our \dyckkm language puts a bound on depth in Dyck-$k$, capturing the long-distance hierarchical structure of natural language as well as its bounded memory requirements.\footnote{For further motivation, we note that center-embedding directly implies bounded memory requirements in arc-eager left-corner parsers \cite{resnik1992left}.}

\section{Related Work}
This work contributes primarily to the ongoing theoretical characterization of the expressivity of RNNs.
 \citet{siegelmann1992computational} proved that RNNs are Turing-complete if provided with infinite precision and unbounded computation time.
Recent work in NLP has taken an interest in the expressivity of RNNs under conditions more similar to RNNs' practical uses, in particular assuming one ``unrolling'' of the RNN per input token, and precision bounded to be logarithmic in the sequence length.
In this setting, \citet{weiss2018practical} proved that LSTMs can implement simplified counter automata; the implications of which were explored by \citet{merrill2019sequential,merrill2020linguistic}.
In this same regime, \citet{merrill2020formal} showed a strict hierarchy of RNN expressivity, proving among other results that RNNs augmented with an external stack \cite{grefenstette2015learning} can recognize hierarchical (Context-Free) languages like Dyck-$k$, but LSTMs and RNNs cannot.

\citet{korsky2019computational} prove that, given infinite precision, RNNs can recognize context-free languages.
Their proof construction uses the floating point precision to simulate a stack, e.g., implementing \texttt{push} by dividing the old floating point value by $2$, and \texttt{pop} by multiplying by $2$.
This implies that one can recognize any language requiring a \textit{bounded} stack, like our \dyckkm, by providing the model with precision that scales with stack depth.
In contrast, our work assumes that the precision cannot scale with the stack depth (or vocabulary size); in practice, neural networks are used with a fixed precision \cite{hubara2017quantized}.

Our work also connects to empirical studies of what RNNs can \textit{learn} given finite samples.
Considerable evidence has shown that LSTMs can learn languages requiring counters (but Simple RNNs do not) \cite{weiss2018practical,sennhauser2018evaluating,yu2019learning,suzgun2019lstm}, and neither Simple RNNs nor LSTMs can learn Dyck-$k$.
In our work, this conclusion is foregone because Dyck-$k$ requires unbounded memory while RNNs have finite memory; we show that LSTMs extrapolate well on \dyckkm, the memory-bounded variant of Dyck-$k$.
Once augmented with an external (unbounded) memory, RNNs have been shown to learn hierarchical languages \cite{suzgun2019lstm,hao2018context,grefenstette2015learning,joulin2015inferring}.
Finally, considerable study has gone into what RNN LMs learn about natural language syntax \cite{lakretz2019emergence,khandelwal2018sharp,gulordava2018colorless,linzen2016assessing}.

\section{Preliminaries and definitions} \label{sec_prelims_definitions}

\subsection{Formal languages}
A formal language $\mathcal{L}$ is a set of strings $\mathcal{L}\subseteq \Sigma^*\omega$ over a fixed vocabulary, $\Sigma$ (with the end denoted by special symbol $\omega$). %
We denote an arbitrary string as $w_{1:T}\in\Sigma^*\omega$, where $T$ is the string length.

The Dyck-$k$ language is the language of nested brackets of $k$ types, and so has $2k$ words in its vocabulary: $\Sigma=\{\symopen_i,\symclose_i\}_{i\in[k]}$.
Any string in which brackets are well-nested, i.e., each $\symopen_i$ is closed by its corresponding $\symclose_i$, is in the language.
Formally, we can write it as the strings generated by the following context-free grammar:
\begin{align*}
  X \rightarrow  &\mid \symopen_i \ \ X \ \ \symclose_i \ \ X\\
  &\mid  \ \epsilon,
\end{align*}
where $\epsilon$ is the empty string.\footnote{And $\omega$ appended to the end.}
The memory necessary to generate any string in Dyck-$k$ is proportional to the number of \textit{unclosed} open brackets at any time.
We can formalize this simply by counting how many more open brackets than close brackets there are at each timestep:
\begin{align*}
  d(w_{1:t}) = \texttt{count}(w_{1:t},\symopen) - \texttt{count}(w_{1:t},\symclose)
\end{align*}
where $\texttt{count}(w_{1:t},a)$ is the number of times $a$ occurs in $w_{1:t}$.
We can now define \dyckkm by combining Dyck-$k$ with a depth bound, as follows:
\begin{restatable}[\dyckkm]{defn}{defn_dyckkm}
\label{defn_dyckkm}
  For any positive integers $k,m$, 
   \dyckkm is the set of strings \begin{align*}
    \{w_{1:T} \in \text{Dyck-$k$}\mid \forall_{t=1,\dots,T}, d(w_{1:t}) \leq m\}
  \end{align*}
\end{restatable}

\subsection{Recurrent neural networks}

We now provide definitions of recurrent neural networks as probability distributions over strings, and define what it means for an RNN to generate a formal language.
We start with the most basic form of RNN we consider, the Simple (Elman) RNN:
\begin{restatable}[Simple RNN (generator)]{defn}{defnsimplernn}
\label{defn_simplernn}
  A Simple RNN (generator) with $d$ hidden units is a probability distribution $f_\theta$ of the following form:
\begin{align*}
  &h_0 = \mathbf{0}\\
  &h_t = \sigma(W h_{t-1} + Ux_t + b)\\
  &w_t | w_{1:t-1} \sim \text{softmax}(g_\theta(h_{t-1}))
\end{align*}
where $h_t\in\mathbb{R}^d$ and the function $g$ has the form $g_\theta(h_{t-1}) = V h_{t-1} + b_v$. The input $x_t =Ew_t$; overloading notation, $w_t$ is the one-hot (indicator) vector representing the respective word. $\theta$ is the set of trainable parameters, $W, U, b, V, b_v, E$.
\end{restatable}

The Long Short-Term Memory (LSTM) model \cite{hochreiter1997long} is a popular extension to the Simple RNN, intended to ease learning by resolving the vanishing gradient problem.
In this work, we're not concerned with learning but with expressivity. 
We study whether the LSTM's added complexity enables it to generate \dyckkm using less memory.
\begin{restatable}[LSTM (generator)]{defn}{defnlstmgenerator}
\label{defn_lstmgenerator}
  An LSTM (generator) with $d$ hidden units is a probability distribution $f_\theta$ of the following form:
\begin{align*}
  &h_0,c_0 = \mathbf{0}\\
  &f_t = \sigma(W_fh_{t-1} + U_fx_t + b_f)\\
  &i_t = \sigma(W_ih_{t-1} + U_ix_t + b_i)\\
  &o_t = \sigma(W_oh_{t-1} + U_ox_t + b_o)\\
  &\tilde{c}_t = \text{tanh}(W_{\tilde{c}} h_{t-1} + U_{\tilde{c}}x_t + b_{\tilde{c}})\\
  &c_t = f_t \odot c_{t-1} + i_t \odot \tilde{c}_t\\
  &h_t = o_t \odot \text{tanh}(c_t)\\
  &w_t | w_{1:t-1} \sim \text{softmax}(g_\theta(h_{t-1}))
\end{align*}
where $h_t,c_t\in\mathbb{R}^d$, the function $g$ has the form $g_\theta(h_{t-1}) = V h_{t-1} + b$, and $x_t = Ew_t$, where $w_t$ is overloaded as above.
$\theta$ is the set of trainable parameters: all $W,U,b$, as well as $V,E$.
\end{restatable}

\paragraph{Notes on finite precision.}
Under our finite precision setting, each hidden unit is a rational number specified using $p$ bits; hence it can take on any value in $\mathbb{P}\subset \mathbb{Q}$, where $|\mathbb{P}|= 2^p$. Each construction is free to choose its specific subset.\footnote{The $\mathbb{P}$ for our constructions is provided in Appendix~\ref{appendix_accounting_of_finite_precision}.} %
A machine with $d$ such hidden units thus can take on any of $2^{dp}$ configurations.

Our constructions require the sigmoid ($\sigma(x)=\frac{1}{1+e^{-x}}$) and tanh nonlinearities to saturate (that is, take on the values at the bounds of their ranges) to ensure arbitrarily long-distance dependencies.
Under standard definitions, these functions approach but never take on their bounding values.
Fortunately, under finite precision, we can provide non-standard definitions under which, if provided with large enough inputs, the functions saturate.\footnote{For example, because the closest representable number (in $\mathbb{P}$) to the true value of $\sigma(x)$ for some $x$ is $1$ instead of some number $<1$.}
Let there be $\beta\in\mathbb{R}$ such that $\sigma(x) = 1$ if $x > \beta$, and $\sigma(x)=0$ if $x<-\beta$. Likewise for hyperbolic tangent, $\text{tanh}(x) = 1$ if $x>\beta$, and $\text{tanh}(x) = {-1}$ if $x < -\beta$.
This reflects empirical behavior in toolkits like PyTorch \cite{paszke2019pytorch}.

\subsection{Formal language generation}

With this definition of RNNs as generators, we now define what it means for an RNN (a distribution) to generate a language (a set).
Intuitively, since a formal language is a set of strings $\mathcal{L}$, our definition should be such that a distribution generates $\mathcal{L}$ if its probability mass on the set of all strings $\Sigma^*\omega$
is concentrated on the set $\mathcal{L}$.
So, we first define the set of strings on which a probability distribution concentrates its mass.  %
The key intuition is to control the local token probabilities $f_\theta(w_t|w_{1:t-1})$, not the global $f_\theta(w_{1:T})$, which must approach zero with sequence length.
\begin{restatable}[locally $\epsilon$-truncated support]{defn}{defnepsilontruncated}
\label{defn_epsilontruncated}
Let $f_\theta$ be a probability distribution over $\Sigma^*\symend$, with conditional probabilities $f_\theta(w_t|w_{1:t-1})$.
  Then the locally $\epsilon$-truncated support of the distribution is the set
  \begin{align*}
    \{w_{1:T} \in \Sigma^*\omega : \forall_{t\in1\dots T}, f_\theta(w_{t}|w_{1:t-1}) \geq \epsilon\}.
  \end{align*}
\end{restatable}
This is the set of strings such that the model assigns at least $\epsilon$ probability mass to each token conditioned on the prefix leading up to that token.
A distribution generates a language, then, if there exists an $\epsilon$ such that the locally truncated support of the distribution is equal to the language:\footnote{We also note that any $f_\theta$ generates multiple languages, since one can vary the parameter $\epsilon$; for example, any softmax-defined distribution must generate $\Sigma^*$ with $\epsilon$ small because they assign positive mass to all strings. }
\begin{defn}[generating a language]
  A probability distribution $f_\theta$ over $\Sigma^*$ generates a language $\mathcal{L}\subseteq \Sigma^*$ if there exists $\epsilon>0$ such that the locally $\epsilon$-truncated support of $f_\theta$ is $\mathcal{L}$.
\end{defn}

\section{Formal results} \label{sec_formal_results}

We now state our results. We provide intuitive proof sketches in the next section, and leave the full proofs to the Appendix.
We start with an application of known work to prove that \dyckkm can be generated by RNNs.
\begin{restatable}[Naive generation of \dyckkm]{thm}{thmnaivegeneration}
\label{thm_naive_generation}
  For any $k,m\in\mathbb{Z}^+$, there exists a Simple RNN $f_\theta$ with $O(k^{m+1})$ hidden units that generates \dyckkm.
 \end{restatable}
The proof follows by first recognizing that there exists a deterministic finite automaton (DFA) with $O(k^{m+1})$ states that generates \dyckkm.
Each state of the DFA is a sequence of up to $m$ unclosed brackets (of $k$ possible types), implying $k^{m+1}-1$ total states.
Then, one applies a general-purpose algorithm for implementing DFAs with $|Q|$ states using an RNN with $O(|Q|)$ hidden units \cite{omlin1996constructing,merrill2019sequential}. Intuitively, this construction assigns a separate hidden unit to each state.\footnote{The construction of \citet{indyk1995optimal} may achieve $O(\sqrt{|Q|})$ in this case (they do not discuss how vocabulary size affects construction size),  but this is still $O(k^{\frac{m}{2}})$ and thus intractable.}

 For our results, we first present two theorems for the Simple RNN and LSTM that use $O(mk)$ hidden units by simulating a stack of $m$ $O(k)$-dimensional vectors, which are useful for discussing the constructions.  Then we show how to reduce to $O(m\log k)$ via an efficient encoding of $k$ symbols in $O(\log k)$ space.

\begin{restatable}[]{thm}{thmsimplernntwomk}
\label{thm_simplernn2mk}
  For any $k,m\in\mathbb{Z}^+$, there exists a Simple RNN $f_\theta$ with $2mk$ hidden units that generates \dyckkm.
\end{restatable}
We state $2mk$ exactly instead of $O(mk)$ because this exactness is interesting and the constant is small; further, we find that the modeling power of the LSTM leads to a factor of 2 improvement:
\begin{restatable}[]{thm}{thmlstmmk}
\label{thm_lstmmk}
  For any $k,m\in\mathbb{Z}^+$, there exists a LSTM $f_\theta$ with $mk$ hidden units and $W_{\tilde{c}}=\mathbf{0}$ that generates \dyckkm.
\end{restatable}
We point out the added property that $W_{\tilde{c}}=\mathbf{0}$ because, informally, this matrix corresponds to the recurrent matrix $W$ of the Simple RNN; it's the only matrix operating on LSTM's memory that \textit{isn't} used in computing a gate.
Thus, the LSTM we provide as proof uses \textit{only} its gates.

Using the same mechanisms as in the proofs above but using an efficient encoding of each stack element in $O(\log k)$ units, we achieve the following.

\begin{restatable}[]{thm}{thmsimplernnmlogk}
\label{thm_simplernnmlogk}
  For any $k,m \in\mathbb{Z}^+$, where $k>1$, 
  there exists a Simple RNN $f_\theta$ with $6m\lceil\log k\rceil-2m$ hidden units that generates \dyckkm.
\end{restatable}
Likewise, for LSTMs, we achieve a more memory-efficient generator. 
\begin{restatable}[]{thm}{thmlstmmlogk}
\label{thm_lstmmlogk}
 For any $k,m\in\mathbb{Z}^+$, where $k>1$, there exists an LSTM $f_\theta$ with $3m\lceil\log k\rceil-m$ hidden units and $W_{\tilde{c}}=\mathbf{0}$ that generates \dyckkm.
\end{restatable}

\paragraph{Note on memory.}
While we have emphasized expressive power under memory constraints---what functions can be expressed, not what is learned in practice---neural networks are frequently intentionally overparameterized to aid learning \cite{zhang2017understanding,shwartz2017opening}.
Even so, known constructions for \dyckkm would require a number of hidden units far beyond practicality.
Consider if we were to use a vocabulary size of $100{,}000$, and a practical depth bound of $3$.
Then if we were using a $k^{m+1}$ hidden unit construction to generate \dyckkm, we would need $100{,}000^4=10^{20}$ hidden units.
By using our LSTM construction, however, we would need only $3\times3\times\lceil\log_2(100{,}000)\rceil -1\times3 = 150$ hidden units, suggesting that networks of the size commonly used in practice are large enough to learn these languages.

\paragraph{Lower bound.}
We also show that the bounds in Theorems~\ref{thm_simplernnmlogk},~\ref{thm_lstmmlogk} are tight. Specifically, the following theorem formalizes the statement that \textit{any} algorithm that uses a $d$-dimensional finite-precision vector memory to generate \dyckkm must use $d \in \Omega(m\log k)$ memory, implying that RNNs are optimal for doing so, up to constant factors.
\begin{restatable}[$\Omega(m\log k)$ to generate \dyckkm]{thm}{thmlowerbound}
\label{thm_lower_bound}
  Let $A$ be an arbitrary function from $d$-dimensional vectors and symbols $w_t\in\Sigma$ to $d$-dimensional vectors; $A : \mathbb{P}^d\times \Sigma \rightarrow \mathbb{P}^d$, $A(h_{t-1},w_t) \mapsto h_t$.
  Let $\psi$ be an arbitrary function from $\mathbb{P}^d$ to probability distributions over $\Sigma\cup\{\symend\}$.
  Let $f_\psi$ be a probability distribution over $\Sigma^*\symend$, with the form $f_\psi(w_{1:T}) = \prod_{t=1}^{T}f(w_t|w_{1:t-1})$, where $f(w_t|w_{1:t-1}) = \psi(h_{t-1})$.
  If $f$ generates \dyckkm, then $d\geq\frac{m\log k}{p}=\Omega(m\log k)$. 
 \end{restatable}
Intuitively, $A$ is an all-powerful recurrent algorithm that represents prefixes $w_{1:t}$ as vectors, and $\psi$, also all-powerful, turns each vector into a probability distribution over the next token.
The proof follows from a simple communication complexity argument:
to generate Dyck-$(k,m)$, any algorithm needs to distinguish between all subsequences of unclosed open brackets, of which there are $k^m$.
So, $2^{dp} \geq k^m$, and the dimensionality $d\geq \frac{m\log k}{p}$. Since $p=O(1)$, we have $d= \Omega(m\log k)$.

\section{Stack constructions in Simple RNNs}
The memory needed to close all the brackets in a \dyckkm prefix $w_{1:t}$ can be represented as a stack of (yet unclosed) open brackets $[\symopen_{i_1},\dots,\symopen_{i_{m'}}]$, $m'\leq m$; reading each new parenthesis either pushes or pops from this stack.
Informally, all of our efficient RNN constructions generate \dyckkm by writing to and reading from an implicit stack that they encode in their hidden states.
In this section, we present some challenges in a naive approach, and then describe a construction to solve these challenges.
We provide only the high-level intuition; rigorous proofs are provided in the Appendix.

\subsection{An extended model}

We start by describing what is achievable with an extended model family, \textit{second-order RNNs} \cite{rabusseau2019connecting,lee1986higher}, which allow their recurrent matrix $W$ to be chosen as a function of the input (unlike any of the RNNs we consider.)
Under such a model, we show how to store a stack of up to $m$ $k$-dimensional vectors in $mk$ memory.
Such a representation can be thought of as the concatenation of $m$ $k$-dimensional vectors in the hidden state, like this:
\begin{center}
  \includegraphics[width=.16\linewidth]{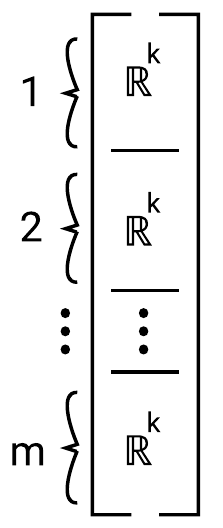}
\end{center}
We call each $k$-dimensional component a ``stack slot''.
If we want the first stack slot to always represent the top of the stack, then there's a natural way to implement \texttt{pop} and \texttt{push} operations.
In a \texttt{push}, we want to shift all the slots toward the bottom, so there's room at the top for a new element.
We can do this with an off-diagonal matrix $W_{\text{push}}$:\footnote{Note that only needing to store $m$ things means that when we push, there should be nothing in slot $m$; otherwise, we'd be pushing element $m+1$.}
\begin{center}
  \includegraphics[width=.62\linewidth]{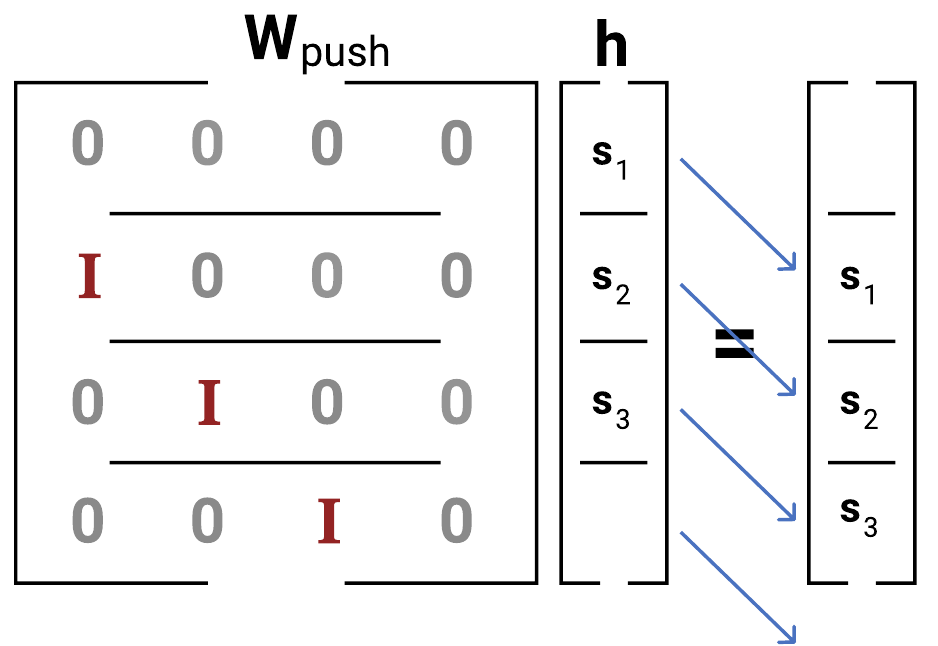}
\end{center}
This would implement the $Wh_{t-1}$ part of the Simple RNN equation.
We can then write the new element (given by $Ux_t$) to the first slot.
If we wanted to pop, we could do so with another off-diagonal matrix $W_{\text{pop}}$, shifting everything towards the top to get rid of the top element:
\begin{center}
  \includegraphics[width=.59\linewidth]{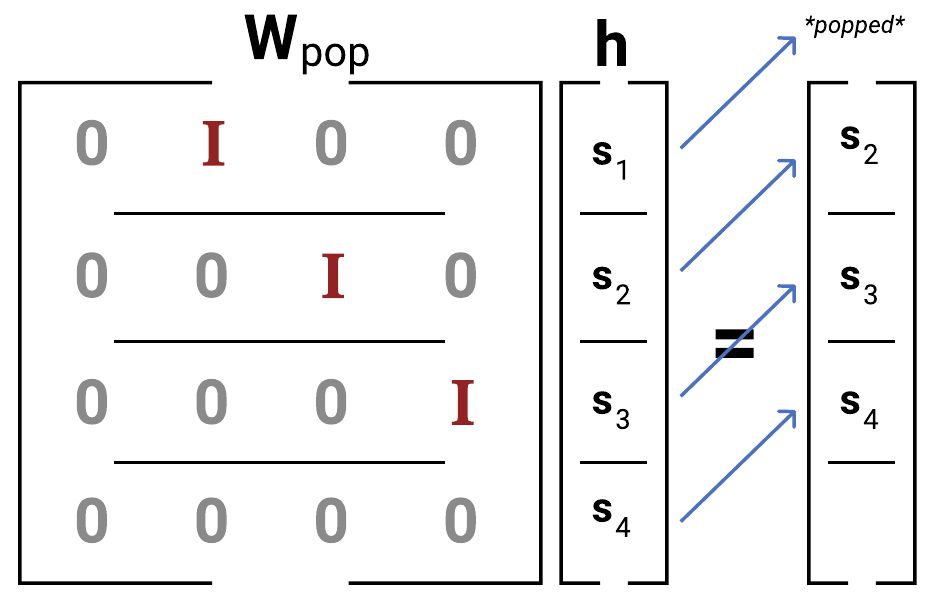}
\end{center}
This won't work for a Simple RNN because it only has one $W$.

\subsection{A Simple RNN Stack in $2mk$ memory}
Our construction gets around the limitation of only having a single $W$ matrix in the Simple RNN by doubling the space to $2mk$.
Splitting the space $h$ into two $mk$-sized partitions, we call one $h_{\text{pop}}$, the place where we write the stack if we see a \texttt{pop} operation, and the other $h_\text{push}$ analogously for the \texttt{push} operation.
If one of $h_\text{pop}$ or $h_\text{push}$ is empty (equal to $0$) at any time, we can try reading from both of them, as follows:
\begin{center}
  \includegraphics[width=.95\linewidth]{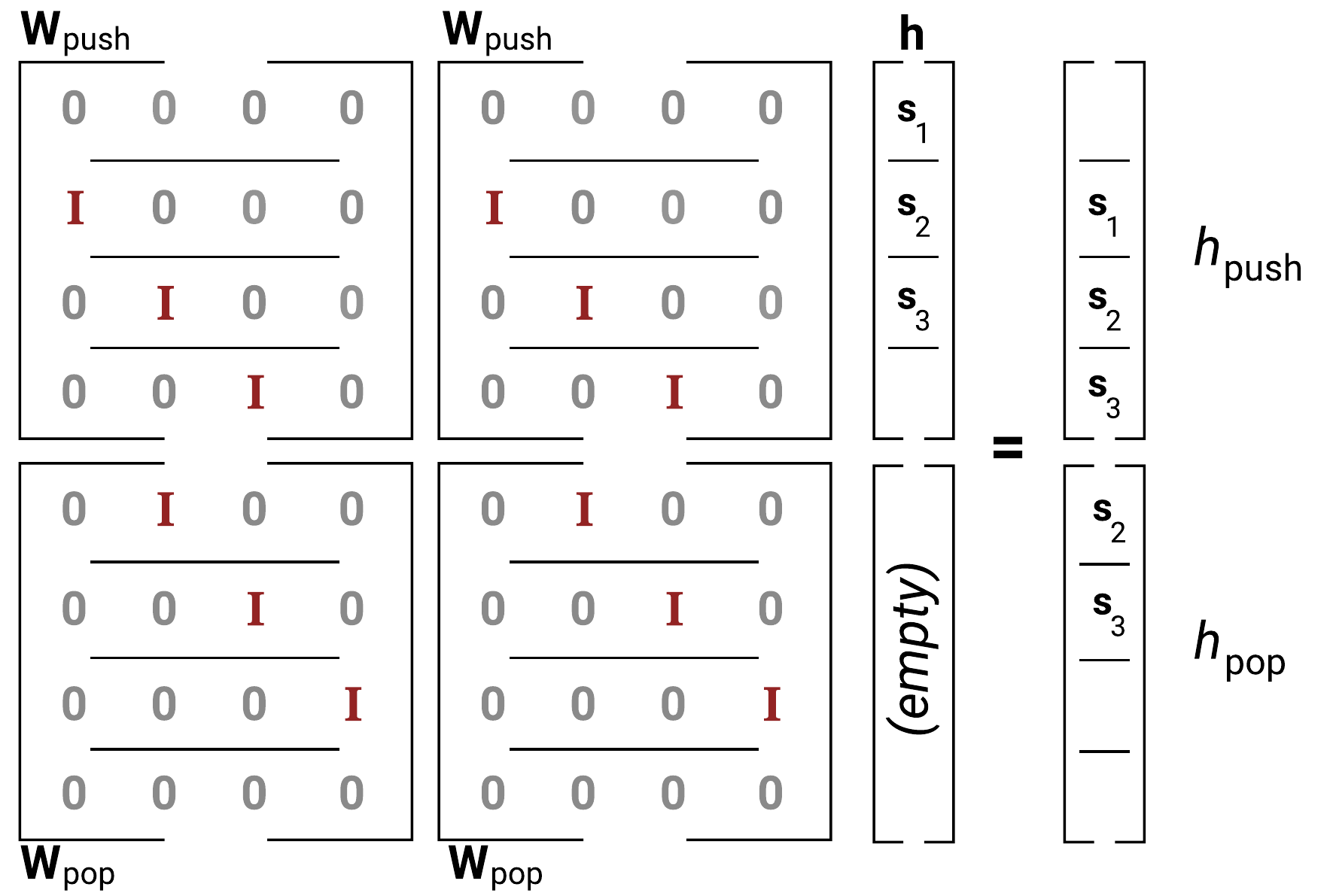}
\end{center}
Our $W$ matrix is actually the concatenation of two of the $W_{\text{pop}}$ and $W_{\text{push}}$ matrices.
Now we have two candidates, both $h_{\text{push}}$ and $h_{\text{pop}}$;  but we only want the one that corresponds to \texttt{push} if $w_t=\symopen_i$, or \texttt{pop} if $w_t=\symclose_i$.
We can zero out the unobserved option using the term $Ux_t$, adding a large negative value to every hidden unit in the stack that doesn't correspond to \texttt{push} if $x_t$ is an open bracket $\symopen_i$, or \texttt{pop} if $x_t$ is a close bracket $\symclose_i$:
\begin{center}
  \includegraphics[width=.72\linewidth]{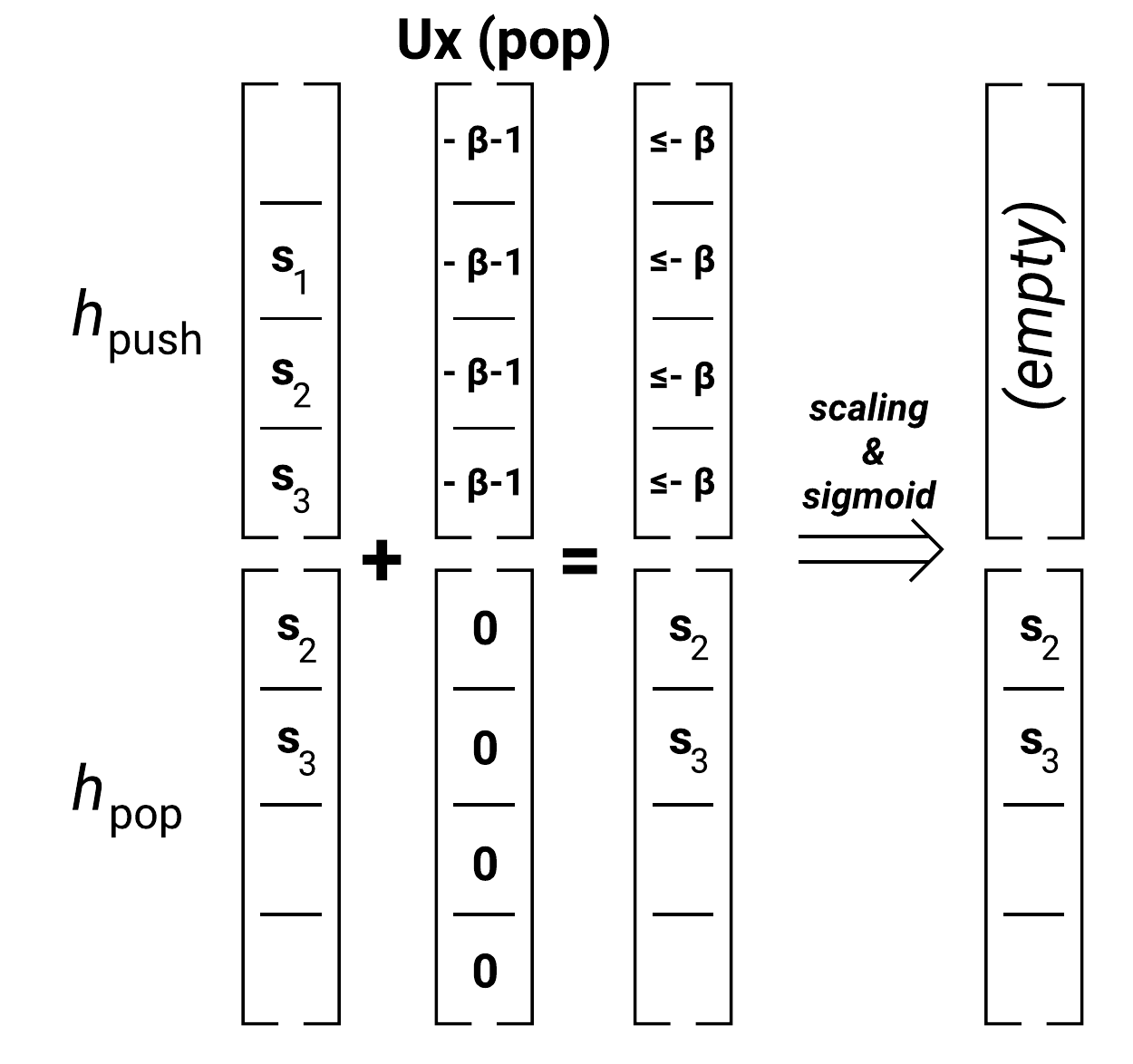}
\end{center}
So $Ux_{\symopen_i}=[e_i;0,\dots,-\beta-1,\dots]$, where $e_i$ is a one-hot representation of $\symopen_i$, and $Ux_{\symclose_i}=[-\beta-1,\dots,-\beta-1,0,\dots]$,
where $\beta$ is the very large value we specified in our finite-precision arithmetic.
Thus, when we apply the sigmoid to the new $Wh_{t-1} + Ux_t + b$, whichever of $\{h_{\text{pop}}, h_{\text{push}}\}$ doesn't correspond to the true new state is zeroed out.\footnote{For whichever of $h_{\text{tmp}}\in \{h_{\text{pop}},h_{\text{push}}\}$ that is not zeroed out, $\sigma(h_{\text{tmp}})\not=h_{\text{tmp}}$. Hence, we scale all of $U$ and $W$ to be large, such that $\sigma(Wh_{t-1})\in\{0,1\}$.} 

\begin{figure*}
  \centering
  \includegraphics[width=.85\linewidth]{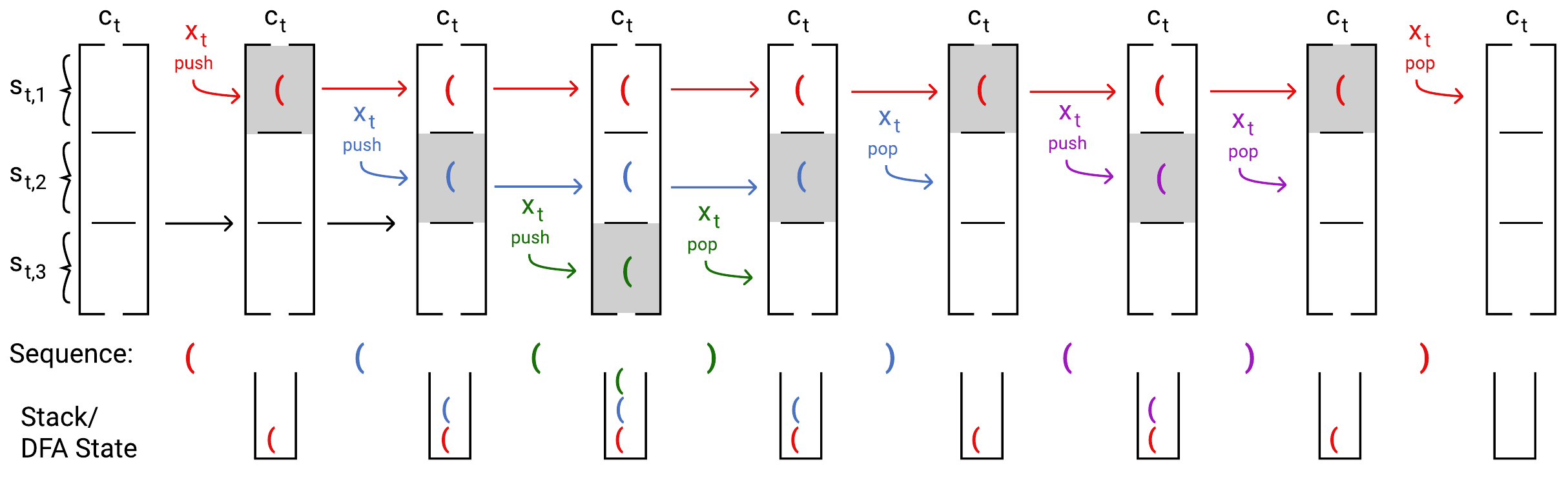}
  \caption{\label{figure_lstm_explanation}Example trace of the hidden state of an LSTM processing a string. The shaded slot is the top of the stack, passed through the output gate to the hidden state.}
\end{figure*}

\section{Stack construction in LSTMs}
We could implement our $2mk$ Simple RNN construction in an LSTM, but its gating functions suggest more flexibility in memory management, and \citet{levy2018long} claim that the LSTM's modeling power stems from its gates.
With an LSTM, we achieve the $mk$ of the oracle we described, all while \textit{exclusively} using its gates.
\looseness=-1

\subsection{An LSTM stack in $mk$ memory}
To implement a stack using the LSTM's gates, we use the same intuitive description of the stack as before: $m$ stack slots, each of dimensionality $k$.
However while the top of the stack is the first slot in the Simple RNN, the bottom of the stack is the first slot in the LSTM.
Before we discuss mechanics, we introduce the memory dynamics of the model.
Working through the example in Figure~\ref{figure_lstm_explanation}, when we push the first open bracket, it's assigned to slot $1$; then the second and third open brackets are assigned to slots $2$ and $3$.
Then a close bracket is seen, so slot $3$ is erased. %
In general, the stack is represented in a contiguous sequence of slots, where the first slot represents the bottom of the stack.
Thus, the top of the stack could be at \textit{any} of the $m$ stack slots.
So to allow for ease of linearly reading out information from the stack, we store the full stack only in the cell state $c_t$, and let only the slot corresponding to the top of the stack, which we'll refer to as the \textit{top slot}, through the output gate to the hidden state $h_t$.

Recall that the LSTM's cell state $c_t$ is specified as $c_t = f_t \odot c_{t-1} + i_t\odot \tilde{c}_t$.
The values $f_t$ and $i_t$ are the forget gate and input gate, while $\tilde{c}_t$ contains information about the new input.\footnote{Note that, since $W_{\tilde{c}}=\mathbf{0}$, we have $\tilde{c}_t = \text{tanh}(U_{\tilde{c}}x_t + b_{\tilde{c}})$, so it does not depend on the history $h_{t-1}$.}
An LSTM's hidden state $h$ is related to the cell state as $h_t = o_t \odot \text{tanh}(c_t)$, where $o_t$ is the output gate.

\paragraph{\texttt{push} mechanics.}
To implement a \texttt{push} operation, the \textbf{input gate} finds the first free slot (that is, equal to $\mathbf{0}\in \mathbb{R}^k$) by observing that it is directly after the top slot of $h_{t-1}$. The input gate is set to $1$ for all hidden units in this first free stack slot, and $0$ elsewhere.
This is where the new element will be written.
The \textbf{new cell candidate} $(\tilde{c}_t)$ is used to \textit{attempt} to write the input symbol $\symopen_i$ to \textit{all} $m$ stack slots.
But because of the input gate, $\symopen_i$ is only written to the first free slot.
The \textbf{forget gate} is set to $1$ everywhere, so the old stack state is copied into the new cell state. %
This is summarized in the following diagram, where the dark grey bar indicates where the gate is set to $0$:
\begin{center}
\includegraphics[width=0.57\linewidth]{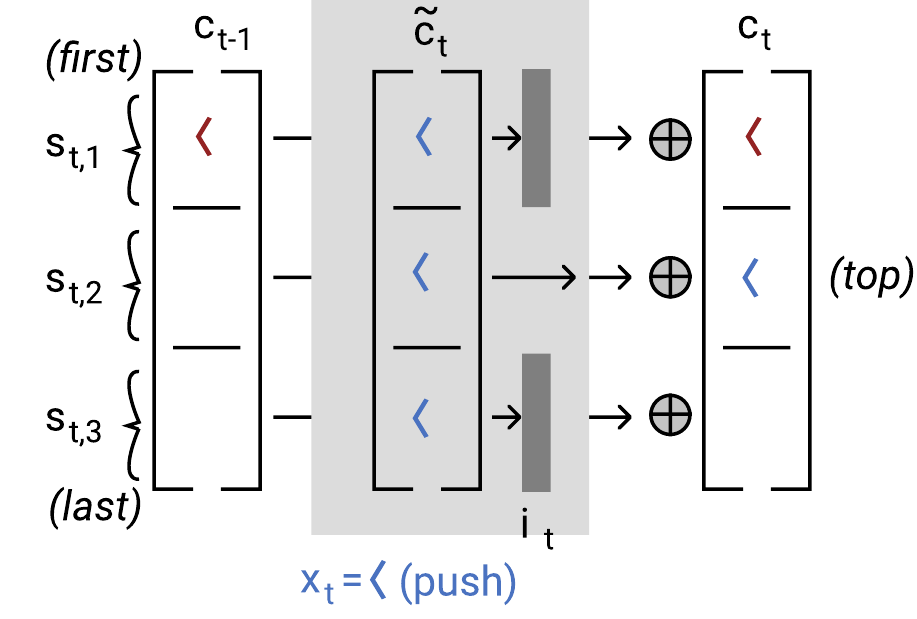}
\end{center}
\paragraph{\texttt{pop} mechanics.}
To implement a \texttt{pop} operation, the \textbf{forget gate} finds  the top slot, which is the slot farthest from slot $1$ that \textit{isn't} empty (that is, that encodes some $\symopen_i$.)
In practice, we do this by guaranteeing that the forget gate is equal to $1$ for all stack slots before (but excluding) the last non-empty stack slot.
Since this last non-empty stack slot encodes the top of the stac,, for it and all subsequent (empty) stack slots, the forget gate is set to $0$.\footnote{The \textbf{input gate} and the \textbf{output gate} are both set to $\mathbf{0}$.}
This erases the element at the top of the stack, summarized in the following diagram:
\begin{center}
\includegraphics[width=0.50\linewidth]{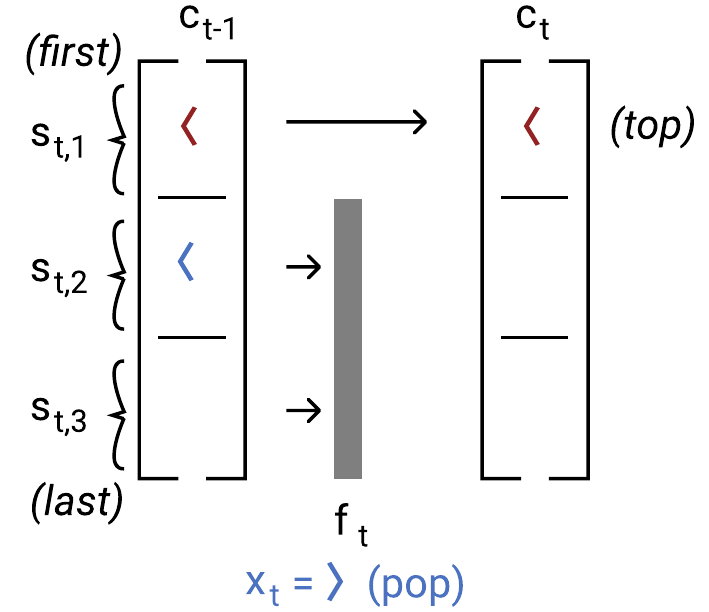}
\end{center}
\paragraph{\texttt{output} mechanics.}
We've so far described how the new cell state $c_t$ is determined.
So that it's easy to tell what symbol is at the top of the stack, we only want the \textit{top slot} of the stack passing through the output gate.
We do this by guaranteeing that the output gate is equal to $0$ for all stack slots from the first slot to the top slot (exclusive). %
The output gate is then set to $1$ for this top slot (and all subsequent empty slots,) summarized in the following diagram:
\begin{center}
\includegraphics[width=0.50\linewidth]{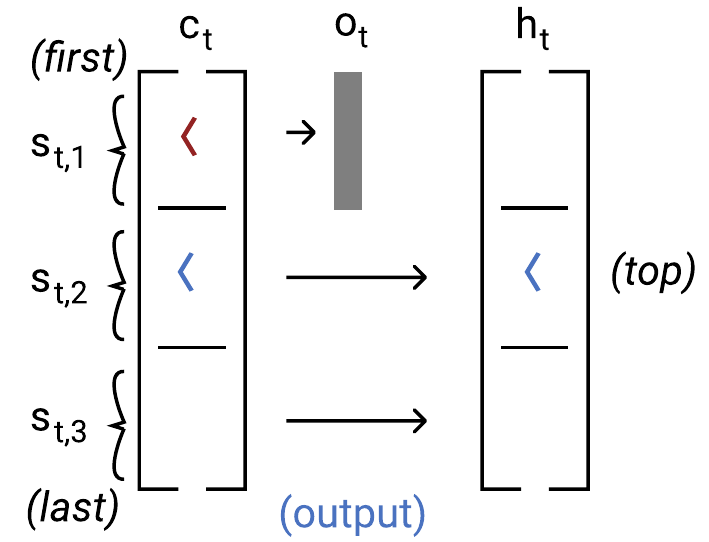}
\end{center}

\section{Defining the generating distribution}
So far, we've discussed how to implement implicit stack-like memory in the Simple RNN and the LSTM\@.
However, the formal claims we make center around RNNs' ability to \textit{generate} \dyckkm.
\subsection{Generation in $O(km)$ memory}
Assume that at any timestep $t$, our stack mechanism has correctly pushed and popped each open and close bracket, as encoded in $h_t,c_t$.
We still need to prove that our probability distribution,
\begin{align}
  w_t \mid w_{1:t-1} \sim \text{softmax}(Vh_{t-1} + b),
\end{align}
assigns greater than $\epsilon$ probability only to symbols that constitute continuations of some string in \dyckkm, by specifying the parameters $V$ and $b_v$.

\paragraph{Observing any open bracket $\symopen$.}
  If and only if fewer than $m$ elements are on the stack, all $\symopen_i$ must be assigned $\geq\epsilon$ probability.
    This encodes the depth bound.
    In our constructions, $m$ elements are on the stack if and only if stack slot $m$ is non-zero.
    So, each row $V_{\symopen_i}$ is zeros except for slot $m$, where each dimension is a large negative number, while the bias term $b_v$ is positive.
    
\paragraph{Observing the end of the string $\omega$.} If and only if $0$ elements are on the stack, the string can end.
    The row $V_{\omega}$ detects if any stack slot is non-empty.\footnote{In particular, the bias term $b_{\omega}$ is positive, but the sum $V_{\omega}h_{t-1} + b_\omega$ is negative if the stack is not empty.}
\paragraph{Observing close bracket $\symclose_i$.} The close bracket $\symclose_i$ can be observed if and only if the top stack slot encodes $\symopen_i$.
Both the Simple RNN construction and the LSTM construction make it clear which stack slot encodes the top of the stack.
In the Simple RNN, it's always the first slot.
In the LSTM, it's the only non-empty slot in $h_t$.
In our stack constructions, we assumed each stack slot is a $k$-dimensional one-hot vector $e_i$ to encode symbol $\symopen_i$.
So in the Simple RNN, $V_\symclose_i$ reads the top of the stack through a one-hot vector $e_i$ in slot $1$, while in the LSTM it does so through $e_i$ in all $m$ slots.
This ensures that $V_{\symclose_i}^\top h_t$ is positive if and only if $\symopen_i$ is at the top of the stack.

\subsection{Generation in $O(m\log k)$ memory}
We now show that $O(\log k)$-dimensional stack slots suffice to represent $k$ symbols.
The crucial difficulty is that we also need to be able to define $V$ and $b$ to be able to linearly detect which $\symopen_i$ is encoded in the top of the stack.\footnote{A further difficulty is guaranteeing that the stack constructions still work with the encoding; due to space, we detail all of this in the Appendix.}
A naive attempt might be treat the $\log k$ hidden units as binary variables, and represent each $\symopen_i$ using the $i^{\text{th}}$ of the $2^{\log k}=k$ binary configurations, which we denote $p^{(i)}$.
This does not work because some $p^{(i)}$ are strict supersets of other $p^{(j)}$, so the true symbol cannot decoded through any $V^\top h_t$.
To solve this, we use a constant factor more space, to ensure each symbol is decodable by $V$. %
In the first $\log k$ units of a stack slot we use the bit configuration $p^{(i)}$.
In the second, we use $(1-p^{(i)})$ (the binary negation.)
Call this encoding $\psi_i \in \mathbb{R}^{2\log k}$.
Using $\psi_i$ for the row $V_{\symclose_i}$, we have the following expression for the dot product in determining the probability distribution:
\begin{align*}
  V_{\symclose_j}^\top \psi_i =& \sum_{\ell=1}^{\lceil \log k\rceil} p^{(i)}_\ell p^{(j)}_\ell + \sum_{\ell=1}^{\lceil\log k\rceil} (1-p^{(i)}_\ell )(1-p^{(j)}_\ell)\\
  &\begin{cases}
    = \lceil\log k\rceil & i=j\\
    \leq \lceil\log k\rceil -1& i\not=j
  \end{cases}
\end{align*}
Thus, we can always detect which symbol $\symopen_i$ is encoded by setting $b_\symclose_i = \log k - 0.5$.\footnote{In actuality, we use a slightly less compact encoding, spending $\log k-1$ more hidden units set to $1$, to incrementally subtract $\log k - 1$ from all logits. Then the bias terms $b_\symclose_i$ are set to $0.5$, avoiding possible precision issues with representing the float $\lceil\log k\rceil - 0.5$.}

\section{Experiments}

\begin{figure}
  \includegraphics[width=\linewidth]{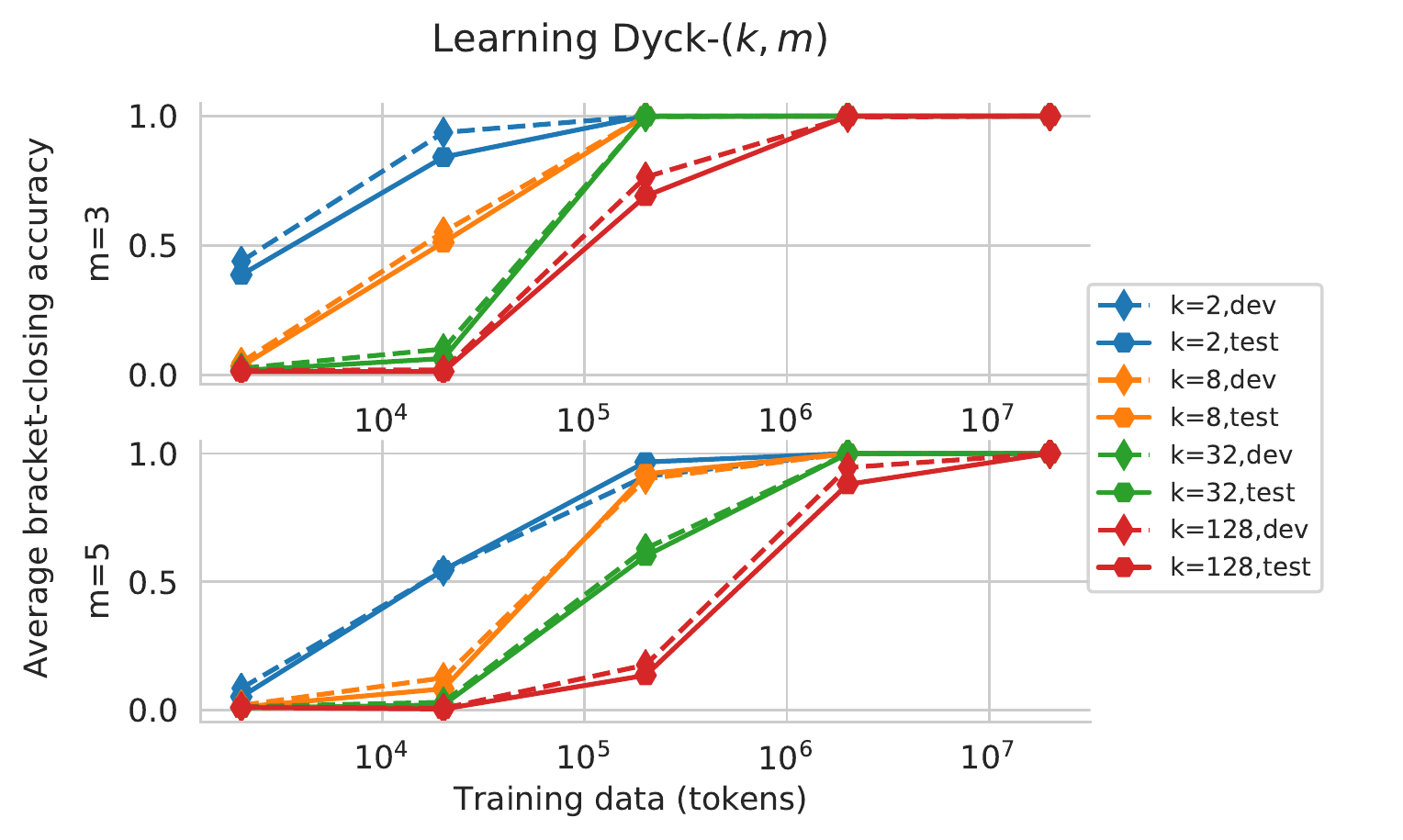}
  \caption{\label{fig_mainpaper_learningdyck} Learning curves for \dyckkm languages.}
\end{figure}
Our proofs have concerned constructing RNNs that generate \dyckkm; now we present a short study connecting our theory to learning \dyckkm from finite samples.
In particular, for $k\in\{2,8,32,128\}$ and $m\in\{3,5\}$, we use our theory to set the hidden dimensionality of LSTMs to $3m\lceil\log k\rceil -m$, and train them as LMs on samples from a distribution\footnote{Defined in Appendix~\ref{appendix_sec_experiments}} over \dyckkm.
For space, we provide an overview of the experiments, with details in the Appendix.
We evaluate the models' abilities to extrapolate to unseen lengths by setting a maximum length of $84$ for $m=3$, and $180$ for $m=5$, and testing on sequences longer than those seen at training time.\footnote{Up to twice as long as the training maximum.}
For our evaluation metric, let $p_j$ be the probability that the model predicts the correct closing bracket given that $j$ tokens separate it from its open bracket.
We report $\text{mean}_j p_j$, to evaluate the model's bracket-closing memory.

For all configurations, we find that the LSTMs using our memory limit achieve error less than $10^{-4}$ when trained on 20 million tokens. %
Strikingly, this is despite the fact that for large $m$ and $k$, a small fraction of the possible stack states is seen at training time;\footnote{See Table~\ref{appendix_table_dfa_percents}.} this shows that the LSTMs are not simply learning $k^m$ structureless DFA states.
Learning curves are provided in Figure~\ref{fig_mainpaper_learningdyck}.

\section{Discussion and conclusion}
We proved that finite-precision RNNs can generate \dyckkm, a canonical family of bounded-depth hierarchical languages, in $O(m\log k)$ memory, a result we also prove is tight.
Our constructions provide insight into the mechanisms that RNNs and LSTMs can implement. %

The Chomsky hierarchy puts all finite memory languages in the single category of regular languages.
But humans generating natural language have finite memory, and context-free languages are known to be both too expressive and not expressive enough \cite{chomsky1959certain,joshi1990convergence}.
We thus suggest the further study of what \textit{structure} networks can encode in their memory (here, stack-like) as opposed to (just) their position in the Chomsky hierarchy. %
While we have settled the representation question for \dyckkm,
many open questions still remain:
What broader class of bounded hierarchical languages can RNNs efficiently generate?
Our experiments point towards learnability; what class of memory-bounded languages are efficiently learnable?
We hope that answers to these questions will not just demystify the empirical success of RNNs
but ultimately drive new methodological improvements as well.

\paragraph{Reproducibility}
Code for running our experiments is available at \url{https://github.com/john-hewitt/dyckkm-learning}. An executable version of the experiments in this paper is on CodaLab at \url{https://worksheets.codalab.org/worksheets/0xd668cf62e9e0499089626e45affee864}.

\paragraph{Acknowledgements}
The authors would like to thank Nelson Liu, Amita Kamath, Robin Jia, Sidd Karamcheti, and Ben Newman. %
JH was supported by an NSF Graduate Research Fellowship, under grant number DGE-1656518.
Other funding was provided by a PECASE award.

\bibliography{emnlp2020}
\bibliographystyle{acl_natbib}

\newpage
\appendix

\input{appendix-content.tex}

\end{document}

%% file: appendix-content.tex
\appendix
\section{Appendix outline}
This Appendix has the following order.
In (\S\ref{appendix_sec_dyckkm}), we provide a definition of \dyckkm equivalent to that in the main text but more useful for our proofs.
In (\S\ref{appendix_sec_preliminaries}), we state preliminary definitions and assumptions, and prove the lower-bound of $\Omega(m\log k)$ hidden units to generate \dyckkm.
In (\S\ref{appendix_sec_simple_rnn}), we formally introduce our Simple RNN stack construction, and prove its correctness in a lemma.
In (\S\ref{appendix_sec_lstm}), we formally introduce our LSTM stack construction, and prove its correctness in a lemma.
In (\S\ref{appendix_sec_generating}), we prove that a linear (+softmax) decoder on the Simple RNN and LSTM hidden states can be used to generate \dyckkm in $O(mk)$ hidden units using a $1$-hot encoding of stack elements.
In this section we also prove that a general RNN construction of DFAs allows for generation of \dyckkm in $O(k^{m+1})$ hidden units.
In (\S\ref{appendix_sec_efficient_generating}), we provide an alternative encoding of elements in our stack constructions for the Simple RNN and LSTM that uses $O(\log k)$ space per element, and prove that our stack constructions still hold using this encoding.
We provide a linear (+softmax) decoder on the hidden states of the Simple RNN and LSTM (when using the $O(\log k)$ stack element encoding) that can be used as a drop-in replacement for the decoder from the $1$-hot representations, thus generating \dyckkm. This proving that both the Simple RNN and LSTM generate \dyckkm in $O(m\log k)$ hidden units.

\begin{table}
\small
\centering
\begin{tabular}{p{0.43\linewidth}p{0.20\linewidth}p{0.20\linewidth}}
\toprule
Definition/Proof & Main paper & Appendix\\
\midrule
\dyckkm & Definition 1 & \hyperref[appendix_def_dyckkm]{Definition}~\ref{appendix_def_dyckkm}\\
Simple RNN generator & Definition~\ref{defn_simplernn} & \hyperref[appendix_def_simple_rnn]{Definition~\ref*{defn_simplernn}}\\
LSTM generator & Definition~\ref{defn_lstmgenerator} & \hyperref[appendix_def_lstm]{Definition~\ref*{defn_lstmgenerator}}\\
Locally $\epsilon$-truncated support & Definition~\ref{defn_epsilontruncated} & \hyperref[appendix_def_locally_epsilon]{Definition}~\ref{defn_epsilontruncated}\\
Fixed-precision setting & & \hyperref[appendix_def_fixed_precision]{Definition}~\ref{appendix_def_fixed_precision}\\
\midrule
Stack correspondence lemma for Simple RNNs & & \hyperref[appendix_lemma_simple_rnn_stack_correspondence]{Lemma}~\ref{appendix_lemma_simple_rnn_stack_correspondence} \\
Stack correspondence lemma for LSTMs & & \hyperref[appendix_lemma_lstm_stack_correspondence]{Lemma}~\ref{appendix_lemma_lstm_stack_correspondence} \\
Simple RNN generates \dyckkm using $O(k^m)$ & Theorem~\ref{thm_naive_generation} &\hyperref[appendix_thm_naive_generation]{Theorem~\ref*{thm_naive_generation}} \\
Simple RNN generates \dyckkm using $2mk$ & Theorem~\ref{thm_simplernn2mk} & \hyperref[appendix_thm_simple_rnn_2mk]{Theorem~\ref*{thm_simplernn2mk}} \\
LSTM generates \dyckkm using $mk$ & Theorem~\ref{thm_lstmmk} & \hyperref[appendix_thm_lstm_mk]{Theorem~\ref*{thm_lstmmk}} \\
Simple RNN generates \dyckkm in $6m\lceil\log k \rceil-2m$ & Theorem~\ref{thm_simplernnmlogk} & \hyperref[appendix_thm_simple_rnn_mlogk]{Theorem~\ref*{thm_simplernnmlogk}} \\
LSTM generates \dyckkm in $3m\lceil\log k \rceil-m$ & Theorem~\ref{thm_lstmmlogk} &\hyperref[appendix_thm_lstm_mlogk]{Theorem~\ref*{thm_lstmmlogk}}\\
Lower bound of $\Omega(m\log k)$ & Theorem~\ref{thm_lower_bound} & \hyperref[appendix_thm_lower_bound]{Theorem~\ref*{thm_lower_bound}} \\
\bottomrule
\end{tabular}
\caption{\label{table_appendix_theorems}Correspondence between and hyperlinks for definitions and theorems in the main paper and the same objects in the appendix.}
\end{table}

\section{The \dyckkm languages} \label{appendix_sec_dyckkm}
To better understand the success of neural networks on natural language syntax, we aim for a formal language that models the unbounded recursiveness of natural language while also reflecting its bounded memory requirements.
We thus introduce the \dyckkm languages, corresponding to sequences of balanced brackets of $k$ types with a maximal number of $m$ unclosed brackets at any point in the sequences (yielding a bound stack depth of $m$ to parse such sentences). 

Though in the main text we defined \dyckkm by intersecting Dyck-$k$ with a language that simply bounds the difference between the number of open brackets and the number of close brackets, here we provide an equivalent definition that will aid in our proofs.
Each \dyckkm language, specified by fixing a value of $m$ and $k$, is defined using a deterministic finite automaton (DFA).
Here, we provide a general description of any \dyckkm DFA.

Formally, we define each language by the deterministic finite automaton $\mathcal{D}_{m,k} = (Q, \Sigma, \delta, F, q_0)$.
The vocabulary $\Sigma$ consists of $k$ types of open brackets: $\{\symopen_i\}_{i=1,\dots,k}$  corresponding closing brackets $\{\symclose_i\}_{i=1,\dots,k}$. %
Strings over the vocabulary are $w_{1:T} \in \Sigma^*\symend$, where $\symend\not\in\Sigma$ is a special symbol representing the end of the sequence.
$\Sigma\cup \{\symend\}$ are collectively referred to as symbols.

This slightly nonstandard requirement allows for a natural connection with language models, which must estimate the probability that a string ends at any given token.\footnote{But crucially, since $\symend\not\in\Sigma$, the language model need not define a distribution \textit{after} $\symend$ is seen; it can only be the last token.}
Overloading notation, we'll also use $\mathcal{D}_{m,k}$ to refer to the language (the set of strings) itself, defined as the strings accepted by the DFA.

\subsection{DFA States, $Q$}
We now define the states $q \in Q$.
First, we define reject state $r$, and accept state $[\symend]$.
Each other state is uniquely identified by a list of open bracket symbols of length up to $m$; thus the full set of states is provided by: %
\begin{align*}
  Q = \{[\symend], r \}
  &\cup \Big\{ [\symopen_{i_1}\symopen_{i_2},\dots,\symopen_{i_{m'}}]\Big\}_{m'\in 1\dots m, i_j \in 1\dots k}
\end{align*}
The number of states is thus $k^{m+1}+1$, where all but two states reflect a list of open brackets.
We will denote each list $[\symopen_{i_1}\symopen_{i_2},\dots,\symopen_{i_{m'}}]$, $0\leq m' \leq m$ as a \textit{stack state} with $m'$ \textit{elements}, and the value of $\symopen_{i_{m'}}$ as the top element of the stack.
We let $q_0 = [] = [\symopen_{i_1}\symopen_{i_2},\dots,\symopen_{i_{m'}}]$, $m'=0$, the empty stack state.

\begin{figure}
  \begin{center}
    \includegraphics[width=.8\linewidth]{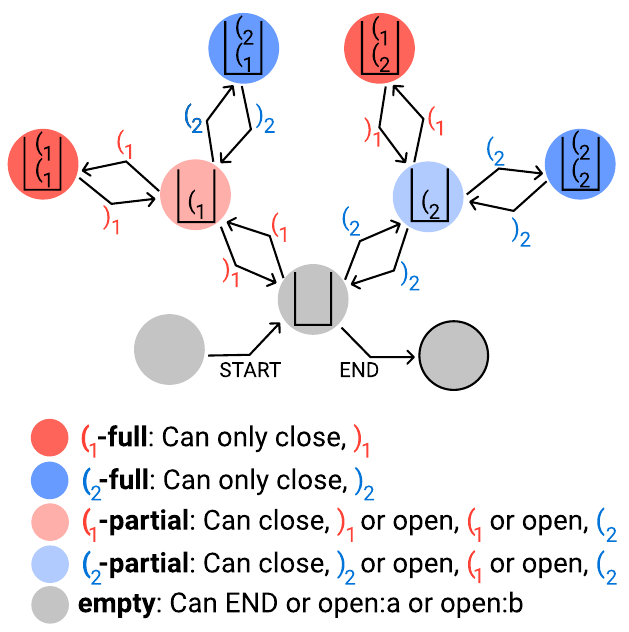}
    \caption{\label{figure_2bounded_dyck2}A deterministic finite automaton describing the transitions of $2$-bounded Dyck-$2$ that do not lead to a reject state (omitted for space) and a qualitative description of the classes of states.}
  \end{center}
\end{figure}

\subsection{Transition function, $\delta$}
We now define the transition function, $\delta$.

\begin{description}
  \item[Empty stack state] The state $[]$ can transition either to the accept state or to another stack state:
    \begin{align}
      &\delta([], \symend) = [\symend]\\
      &\delta([], \symopen_i) = [\symopen_i]
    \end{align}
    while any other symbol transitions to the reject state.
  \item[Partial list states]
    For any state of the form $[\symopen_{i_1}\symopen_{i_2},\dots,\symopen_{i_{m'}}]$, where $m'<m$, an open bracket can be pushed to the list (since $m'<m$), or the last open bracket can be removed, by observing its corresponding close bracket.
    \begin{align}
      \delta([\symopen_{i_1}\symopen_{i_2},\dots,\symopen_{i_{m'}}], \symopen_i) = [\symopen_{i_1}\symopen_{i_2},\dots,\symopen_{i_{m'}}, \symopen_i]\\
      \delta([\symopen_{i_1}\symopen_{i_2},\dots,\symopen_{i_{m'}}], \symclose_{i_{m'}}) = [\symopen_{i_1}\symopen_{i_2},\dots,\symopen_{i_{m'-1}}]
    \end{align}
    All other symbols transition to the reject state.
  \item[Full list states] For any state of the form $[\symopen_{i_1}\symopen_{i_2},\dots,\symopen_{i_m}]$, that is, the list is of length $m$ and thus full, the close bracket of the top of the stack removes that bracket,
    \begin{align}
      \delta([\symopen_{i_1}\symopen_{i_2},\dots,\symopen_{i_m}], \symclose_{i_m}) = [\symopen_{i_1}\symopen_{i_2},\dots,\symopen_{i_{m-1}}]
    \end{align}
    All other symbols transition to the reject state.

  \item[$r$] The reject state $r$ transitions to itself for all symbols.
  \item[{$[\symend]$}] No transitions from the accept state need be defined, since only $\delta([],\symend)=[\symend]$, and $\symend$ must be the last symbol of any string in the universe, and can only occur once.
\end{description}

This accounts for the transition function from all states in $Q$, and completes our definition of the DFA $D_{m,k}$.
We are now prepared to formally define the language:
\begin{defn}[\dyckkm] \label{appendix_def_dyckkm}
  For any $k,m\in\mathbb{Z}^{+}$, The language \dyckkm is the set of strings accepted by the DFA $D_{m,k}$; overloading notation: $D_{m,k} \subseteq \Sigma^*\symend$.\footnote{Where $Z^{+}$ denotes the positive integers.}
\end{defn}

Overloading notation, for any $w_{1:T}\in\Sigma^*\omega$, for any $t\in1,\dots,T$, we'll say $q_t = D_{m,k}(w_{1:t})$, denoting the state that $D_{m,k}$ is in after consuming $w_{1:t}$.

\section{Preliminaries}\label{appendix_sec_preliminaries}

\defnepsilontruncated* \label{appendix_def_locally_epsilon}
Through the notion of $\epsilon$-truncated support, a language model specifies which tokens are allowable continuations of each string prefix by assigning them greater than $\epsilon$ probability.
With this, we're ready to connect RNNs and formal languages:
\begin{defn}[generating a language]
  A probability distribution $f_\theta$ generates a formal language $\mathcal{L}$ if there exists $\epsilon>0$ such that the $\epsilon$-truncated support $\mathcal{L}_{f_\theta}$ is equal to $\mathcal{L}$.
\end{defn}

\subsection{Technical considerations}
\begin{defn}[fixed-precision setting] \label{appendix_def_fixed_precision}
  For language parameters $m,k$ of $D_{m,k}$, and input sequence lengths $T$, we assume that each floating-point value can be specified in $p\in O(1)$ bits (that is, not scaling in any parameter.) These bits choose elements of the set $\mathbb{P}\subset \mathbb{Q}$; we do not specify how $\mathbb{P}$ must be set.
  The $\mathbb{P}$ chosen for our constructions is described in Appendix~\ref{appendix_accounting_of_finite_precision}, after the relevant constants have been defined.
\end{defn}

Under our finite-precision setting, a reasonable assumption about the properties of the sigmoid and hyperbolic tangent functions make our claim considerably simpler.
The sigmoid function, $\sigma(x) = \frac{1}{1+e^{-x}}$, has range $(0,1)$, excluding its boundaries $\{0,1\}$.
However, in a finite-precision arithmetic, $\sigma(x)$ cannot become arbitrarily close to $0$ or $1$.
In fact, in popular deep learning library PyTorch\footnote{Under the \texttt{float} datatype, PyTorch v1.3.1 \cite{paszke2019pytorch}.}, $\sigma(x)$ is exactly equal to $1$ for all $x>6$.
We define the floating point sigmoid function to equal $0$ or $1$ if the absolute value of its input is greater or equal in absolute value to some threshold $\beta$:
\begin{align}
  \sigma_{\text{fp}}(x) = \begin{cases}
    \sigma(x) & -\beta < x < \beta \\
    1 & x \geq \beta\\
    0 & x \leq -\beta,
  \end{cases}
\end{align}
Similarly for the hyperbolic tangent function, we define:
\begin{align}
  \text{tanh}_{\text{fp}}(x) = \begin{cases}
    \text{tanh}(x) & -\beta < x < \beta \\
    1 & x \geq \beta\\
    -1 & x \leq -\beta,
  \end{cases}
\end{align}
For the rest of this paper, we will refer to $\sigma_{\text{fp}}$ as $\sigma$, and $\text{tanh}_{\text{fp}}$ as $\text{tanh}$.

\subsection{Lower bound of $\Omega(m\log k)$ for generation of \dyckkm}
\thmlowerbound* \label{appendix_thm_lower_bound}
\paragraph{Proof.}
We provide a communication complexity argument.
Assume for contradiction that there exists such a machine with $d<\frac{m\log k}{p}$.
Consider any string $w_{1:m}=\symopen_{i_1},\dots,\symopen_{i_{m}}$ of $m$ open brackets, where each is one of $k$ types.
There are $k^{m}$ such strings, but because $d < \frac{m\log k }{p}$, we have that the model has $2^{dp} < k^m$ possible representations, so at least two such strings must share the same representation.
Let such a pair be $w\not=w'$, where likewise $w'$ is defined as $\symopen_{i'_1},\dots,\symopen_{i'_m}$.
Consider the sequence of open bracket indices that defines $w$: $i_1, \dots, i_m$.
Let $s$ be the string $\symclose_{i_m},\dots,\symclose_{i_1}$ of close brackets in the reverse order of the indices of $w$, and likewise $s'$ for $w'$.

The string $w::s$, where $::$ denotes concatenation, is in \dyckkm.
However, $w'::s$ is not in \dyckkm as it breaks the well-nested brackets condition wherever $\symopen_{i_j}\not = \symopen_{i'_j}$.
This must occur at least once, since $w\not=w'$.
Since the model assigns them the same representation even though they must be distinguished to generate \dyckkm, the model does not generate \dyckkm.

\section{Proving Simple RNN stack correspondence in $2mk$ units} \label{appendix_sec_simple_rnn}

In this section, we provide a formal description of our Simple RNN stack construction, and introduce and prove the \textit{stack correspondence lemma}, to guarantee its correctness.
\defnsimplernn* \label{appendix_def_simple_rnn}

\paragraph{Encoding a Stack in the Hidden State}
We define mappings between $Q$---the DFA state space---and the RNN's state space $\mathbb{R}^{2km}$.
First, we encode a stack into a $km$-dimensional vector as follows:
\begin{align}\label{def:srnn-r-map}
    R([\symopen_{i_1}, &\symopen_{i_2}, \dots, \symopen_{i_{m'}}]) \nonumber \\
    &= [e_{i_1}, e_{i_2}, \dots, e_{i_{m'}}, {\bf 0}_{\in \mathbb{R}^{(m-m')k}}]
\end{align}
where $[\dots]$ denotes concatenation of vectors into a $km$-dimensional vector.
We note that $R$ is injective, that is, the inverse map $R^{-1}$ is well-defined on all vectors in the image of $R$.

Later, when constructing the efficient $O(m  \operatorname{log} k)$ memory encoding in Section~\ref{sec:logk-srnn}, we will swap the one-hot vectors $e_i$ for other vectors that will also have entries in $\{0,1\}$.\footnote{Looking ahead to the more complex LSTM construction, we'll introduce notation to denote these vectors $e_i = \psi^{-1}(\symopen_i)$, and look to replace $\psi$ when we develop the $O(m\log k)$ construction.}

Based on this, $\mathcal{S}$ maps a pair of a stack operation ($\sigma \in \{push, pop\}$) and a DFA state $q$ (assuming $q \not\in \{r, [\symend]\}$) to a hidden state in $\mathbb{R}^{2km}$:
\begin{equation}
\mathcal{S}(\sigma, q) = \begin{cases} [R(q), {\bf 0}_{\in \mathbb{R}^{km}}] &\sigma=push\\
[{\bf 0}_{\in \mathbb{R}^{km}}, R(q)] &\sigma=pop
\end{cases}
\end{equation}
Let $\mathcal{R} := Im(S) \subset \mathbb{R}^{2km}$. 
By writing a hidden state $h \in \mathcal{R}$ as $h = [h_{push}, h_{pop}]$, we can obtain the corresponding DFA state by
\begin{equation}\label{def:srnn-q}
    \mathcal{Q}(h) = R^{-1}\left(h_{push} + h_{pop}\right)
\end{equation}
If $h = \mathcal{S}(\sigma, q)$, then the definition guarantees that $\mathcal{Q}(h)=q$.

\paragraph{Defining Transition Matrices}

To define the RNN (apart from $V$, $b_v$), we have to specify the matrices $W, U, E,$, and the vector $b$.

We start by defining $W \in \mathbb{R}^{2km \times 2km}$.
Let $M_{pop}, M_{push} \in \mathbb{R}^{km \times km}$ be the matrices where 
\begin{align} \label{m_up}
  (M_{pop})_{i, j} = \begin{cases}
    2\beta & i=j+k\\
    0 & o.w.
  \end{cases}
\end{align}
\begin{align} \label{m_down}
  (M_{push})_{i, j} = \begin{cases}
    2\beta & i=j-k\\
    0 & o.w.
  \end{cases}
\end{align}
Let $I_{km}$ be the $km\times km$ identity matrix.
Take $W \in \mathbb{R}^{2km \times 2km}$ defined as
\begin{equation}
W = 
\begin{pmatrix}
M_{push}\\
M_{pop}
\end{pmatrix}
\cdot
\begin{pmatrix}
I_{km} & I_{km}
\end{pmatrix}
\end{equation}

To specify the matrices $E$ and $U$, we need to fix an assignment to the integers $1, \dots, 2k$ to the symbols in $\Sigma$. We will assign the integers $1, \dots, k$ to the opening brackets (i.e., $\symopen_1 \mapsto 1, \dots, \symopen_k \mapsto k$), and the integers $1, \dots, k$ to the closing brackets (i.e., $\symclose_1 \mapsto k+1, \dots, \symclose_k \mapsto 2k$). %

Define $E := I_{2k\times 2k}$, and define $U \in \mathbb{R}^{2km\times 2k}$ as:
\begin{equation}
    \begin{pmatrix}
2\beta I_{k\times k} & -2\beta \cdot {\bf 1}_{k\times k}  \\
{\bf 0}_{(m-1)k\times k}             & -2\beta \cdot {\bf 1}_{(m-1)k\times k} \\
-2\beta \cdot {\bf 1}_{k\times k}            & {\bf 0}_{k\times k}  \\
-2\beta \cdot {\bf 1}_{(m-1)k\times k}            & {\bf 0}_{(m-1)k\times k}  \\
    \end{pmatrix}
\end{equation}
Here, we write $I_{k\times k}$ for the identity matrix, and ${\bf 1}_{k\times k}$ for a matrix filled entirely with ones.

Finally, we define
\begin{equation}
    b := -\beta \cdot {\bf 1}_{\mathbb{R}^{2km}}
\end{equation}

\paragraph{Proof of Correctness} 

To prove correctness of the construction, the first step is to show that the transition dynamics of the Simple RNN correctly simulates the dynamics of the stack when consuming one symbol.

Let $\phi(\symopen_i) = push$, $\phi(\symclose_i) = pop$.

\begin{lemma}[One-Step Lemma]
Assume that the hidden state $h_{t}$ encodes a stack $q$:
\begin{equation}
\mathcal{Q}(h_t) = q
\end{equation}
Let $w_{t+1}$ be the new input symbol, and assume that $w_{1:t+1}$ is a valid prefix of a word in \dyckkm.
Let $h_{t+1}$ be the next hidden state, after reading $w_{t+1}$.
Then
\begin{equation}
\mathcal{S}(\phi(w_{t+1}), \delta(q, w_{t+1})) = h_{t+1}
\end{equation}
\end{lemma}
Before showing this lemma, we note the following property of the activation function $\sigma$, under the finite precision assumption:
\begin{lemma}\label{lemma:sigma-binary}
If $u \in \{0,1\}$, then
\begin{align}
    \sigma\left(\beta \cdot (2u-1)\right) &= u \\
    \sigma\left(\beta \cdot (2u-3)\right) &= 0
\end{align}
\end{lemma}
\begin{proof}
By calculation.
\end{proof}

\begin{proof}[Proof of the Lemma]
We can write $h_t = [h_{push}, h_{pop}]$.
Set $h' = h_{push} + h_{pop}$.
By construction of $\mathcal{S}$,
\begin{equation}
    h' = R(q)
\end{equation}
We will write $h' = [h'_1, \dots, h'_m]$, where $h'_i \in \mathbb{R}^{k}$.

At this point, we note that the only relevant property of the encodings $h'_i$ for this proof is that their entries are in $\{0,1\}$, not that they are one-hot vectors. This will make it possible to plug in a more efficient $O(\log k)$ encoding later in Section~\ref{sec:logk-srnn}.

With this, we can write
\begin{align}
Wh_{t} = 
\begin{pmatrix}
M_{push} h' \\
M_{pop} h'
\end{pmatrix}
=
\begin{pmatrix}
2\beta [{\bf 0}_{\in \mathbb{R}^{km}}, h'_1, \dots, h'_{m-1}] \\
    2\beta [h'_2, \dots, h'_m, {\bf 0}_{\in \mathbb{R}^{km}}]
    \end{pmatrix}
\end{align}

\paragraph{Case 1:} Assume $w_{t+1} = \symopen_i$.
We have
\begin{equation}
    Ux_t = [2 \beta e_i, {\bf 0}_{(m-1)k}, -2\beta {\bf 1}_{mk}]
\end{equation}
so that $W h_{t-1} + Ux_t+b$ equals
\begin{equation}
    \begin{pmatrix}
    2\beta [e_i, h'_1, \dots, h'_{m-1}] - \beta {\bf 1} \\
    2\beta [h'_2, \dots, h'_{m-1}, h'_m, {\bf 0}] -3\beta {\bf 1}
    \end{pmatrix}
\end{equation}
Then, by Lemma~\ref{lemma:sigma-binary}, $h_{t+1}$ equals
\begin{equation}
    \begin{pmatrix}
    [e_i, h'_1, \dots, h'_{m-1}] \\
    {\bf 0}_{km}
    \end{pmatrix}
\end{equation}
which is equal to $S(push, \delta(q, \symopen_i))$.

\paragraph{Case 2:} Assume $w_{t+1} = \symclose_i$.
We have
\begin{equation}
    Ux_t = [-2\beta {\bf 1}_{mk}, {\bf 0}_{mk}]
\end{equation}
so that $W h_{t-1} + Ux_t+b$ equals
\begin{equation}
    \begin{pmatrix}
    2\beta [{\bf 0}, h'_1, \dots, h'_{m-1}] - 3 \beta {\bf 1} \\
    2\beta [h'_2, \dots, h'_{m-1}, h'_m, {\bf 0}] -\beta {\bf 1}
    \end{pmatrix}
\end{equation}
Then, $h_{t+1}$ equals
\begin{equation}
    \begin{pmatrix}
    {\bf 0}_{km} \\
    [h'_2, \dots, h'_{m}, {\bf 0}] \\
    \end{pmatrix}
\end{equation}
which is equal to $S(push, \delta(q, \symclose_i))$.
\end{proof}

From the previous lemma, we can derive the following Stack Correspondence lemma, which asserts that the Simple RNN correctly reflects stack dynamics over an entire input string:
\begin{lemma}[Stack Correspondence] \label{appendix_lemma_simple_rnn_stack_correspondence}%
  For all strings $w_{1:T}$ in \dyckkm, for all $t=1,...,T$,
  let $q_t$ be the state of $D_{m,k}$ after consuming prefix $w_{1:t}$.
  Let $h_t$ be the RNN hidden state after consuming $w_{1:t}$.
  Then if $q_t \not \in \{[\symend]\}$, 
  \begin{align}\label{eq:srnn-first-claim}
    \mathcal{Q}(h_t) = q_t,
  \end{align}
\end{lemma}

\begin{proof}
The claim is shown by induction over $t$, applying the previous lemma in each step.
We will show the claim for all $t= 0, 1, \dots, T$, setting $h_0$ to be the  zero vector ${\bf 0}_{2km}$.

As $\mathcal{Q}(h_0) = q_0$, this proves the claim for $t=0$.
To prove the inductive step, we assume $\mathcal{Q}(h_t) = q_t$ has already been shown.

Then, by the preceding One-Step Lemma,
\begin{equation}
\mathcal{S}(\phi(w_{t+1}), \delta(q_t, w_{t+1})) = h_{t+1}
\end{equation}
Noting $q_{t+1} = \delta(q_t, w_{t+1})$, we obtain
\begin{equation}\label{eq:srnn-claim-s}
\mathcal{S}(\phi(w_{t+1}), q_{t+1}) = h_{t+1}
\end{equation}
As noted in the definition of $\mathcal{Q}$ above (\ref{def:srnn-q}), this entails
\begin{equation}
\mathcal{Q}(h_{t+1}) = q_{t+1}
\end{equation}
concluding the inductive step.

\end{proof}

\section{Proving LSTM stack correspondence in $mk$ units} \label{appendix_sec_lstm}

In this section, we prove our lemma describing the stack construction implemented by an LSTM.
To do so, we first define the LSTM:
\defnlstmgenerator* \label{appendix_def_lstm}

First, we provide notation for describing the structure of the memory cell ($c_t$) that $f_\theta$ will maintain.
Next, we proceed to describe precisely, but without referring to the LSTM equations, how this memory changes over time.
We then describe how these high-level dynamics are implemented by the LSTM equations, assuming that we're able to set the values of the intermediate values (gates and new cell candidate) as desired.
We then explicitly construct LSTM weight matrices that provide the desired intermediate values.
This completes the proof of the memory dynamics, which we formalize in a \textit{stack correspondence lemma} once we've introduced the notation.

\subsection{Description of stack-structured memory} 
Our LSTM construction operates by constructing, in its cell state, a representation of the list that defines the DFA stack state.
To make this idea formal, we define a mapping from $\mathbb{R}^d$, the space of the LSTM cell state, to $Q$, the DFA state space.

To do so, we need some language for discussing the LSTM cell states.
For a cell state $c_t\in \mathbb{R}^{mk}$, we define a list of $m$ \textit{stack slots}, $s_{t,1},...,s_{t,m}$, where $s_{t,j} \in \mathbb{R}^k$ is given by dimensions 
\begin{align}
  s_{t,j} = c_{t}[k(j-1)+1:kj+1]
\end{align}

Intuitively, each stack slot will correspond to one item in the list defining a stack state in the DFA $D_{m,k}$, or be empty.
We formalize this by defining a function from stack slots $s_{t,j}\in\mathbb{R}^d$ to elements:
\begin{align}
  \psi(s_{t,j}) = \begin{cases}
    \symopen_i & s_{t,j} = e_i\\
    \varnothing & s_{t,j} = \mathbf{0}
  \end{cases}
\end{align}
where $e_i\in\mathbb{R}^k$ is the $i^{th}$ standard basis vector (one-hot vector) of $\mathbb{R}^k$.
We'll also use the inverse map, $\psi^{-1}$, to map from elements to vectors in $\mathbb{R}^{k}$.

Applying this function to each stack slot allows us to map from $s_{t,1},\dots,s_{t,m}$ to stack states:
\begin{align}\label{eqn_stack_map}
  \mathcal{Q}(s_{t,1},\dots,s_{t,m}) = [\psi(s_{t,1}),\dots,\psi(s_{t,m'})]
\end{align}
where $m'\leq m$ is the maximal integer such that $\psi(s_{t,m'}) \not = \varnothing$.
Intuitively, this means we filter out any empty stack slots at the end of the sequence (those slots $s_{t,j}$ where $j>m'$.)
We'll also make use of the inverse, $\mathcal{Q}^{-1}$, mapping DFA stack states to stack slot lists,
\begin{align*}
  \mathcal{Q}^{-1}&([\symopen_{i_1}, \dots, \symopen_{i_{m'}}]) \\
 &= \mathcal{Q}^{-1}([\symopen_{i_1}, \dots, \symopen_{i_{m'}}, \varnothing,\dots,\varnothing]) \\
  &=\psi^{-1}(\symopen_{i_1}),\dots,\psi^{-1}(\symopen_{i_{m'}}),\mathbf{0}, \dots, \mathbf{0}\\
  &=e_{i_1},\dots,e_{i_{m'}},\mathbf{0}, \dots, \mathbf{0}
\end{align*}
where slots $s_{m'+1},\dots,s_{m}$ are the zero vector $\mathbf{0}$ because only the first $m'$ slots encode a symbol $\symopen_i$, and $\psi(\mathbf{0})=\varnothing$.

\paragraph{On the $\psi$ function}
These one-hot encodings ($e_i$) of symbols ($\symopen_i$) allow for clear exposition, but are not fundamental to the construction; in (\S~\ref{appendix_sec_efficient_generating}), we replace these $k$-dimensional one-hot encodings given by $\psi$ with $O(\log k)$-dimensional encodings.
Throughout our description of the stack construction, we'll explicitly rely on the following property, so we can easily replace this particular $\psi$ in (\S~\ref{appendix_sec_efficient_generating}) by providing another $\psi'$ with the same property.
\begin{defn}[encoding-validity property] \label{appendix_def_encoding_validity}
  Let $\psi'$ be a function $\psi': \{\symopen_i\}_{i\in[k]} \rightarrow \mathbb{R}^n$ mapping open brackets to vectors in $\mathbb{R}^n$.
  Then $\psi'$ obeys the encoding-validity property if for all $i=1,\dots,k$,
  \begin{align}
    \sum_{\ell=1}^{n} \psi^{-1}(\symopen_i)_{\ell} = 1
  \end{align}
  that is, the sum of all dimensions of the representation is equal to $1$, and
  \begin{align}
    \forall_{i\in[k]}\psi^{-1}(\symopen_i) \in \{0,1\}^n
  \end{align}
  that is, the representation takes on values only in $0$ and $1$, and
  \begin{align}
    \psi(\mathbf{0}) = \varnothing,
  \end{align}
  that is, the empty stack slot is encoded by the zero vector, and
  \begin{align}
    &\forall_{i,j\in[k]}, i \not = j \implies \psi^{-1}(\symopen_i) \not = \psi^{-1}(\symopen_j)\\
    &\forall_{i\in[k]}, \psi^{-1}(\symopen_i) \not = \mathbf{0}
  \end{align}
    that is, encodings of symbols $\symopen_i$ are unique, and none are equal to to the encoding of the empty symbol $\varnothing$.
\end{defn}

The encoding we've so far provided, $\psi^{-1}(e_i) = \symopen_i$, obeys these properties: the sum of a single one-hot vector is $1$, no two such vectors are the same, one-hot vectors only take on values in $\{0,1\}$, and we let $\psi(\mathbf{0})=\varnothing$.

\subsection{Description of stack state dynamics} 

Before we describe the memory dynamics, we introduce a useful property of the stack slots.
\begin{defn}[$j$-top]
  A cell state $c_t$ composed of stack slots $s_{t,1},\dots,s_{t,m}$ is $j$-top if there exists $i\in\{1,\dots,k\}$ such that $s_{t,j}=\psi^{-1}(\symopen_i)$, and for all  $j'\in\{j+1,\dots,m\}$, $s_{t,j'} = \mathbf{0}$.
\end{defn}
Intuitively, we'll enforce the constraint that a cell state that is $j$-top \textbf{encodes the element $\symopen_{m'}$ at the top of the stack $[\symopen_i,\dots,\symopen_{m'}]$ in stack slot $j=m'$.}

We're now ready to describe the memory dynamics of $f_\theta$.
The memory $c_t$ is initialized to $\mathbf{0}$. So for all $j\in\{1,\dots,m\}$,
\begin{align}
  &s_{0,j} = \mathbf{0}
\end{align}
The recurrent equation of stack slots is then
\begin{align}\label{eqn_stack_recurrence}
  &s_{j,t} = \begin{cases}
    \mathbf{0} & c_{t-1} \text{ is } j\text{-top, and } w_t = \symclose_i\\
    \psi^{-1}(\symopen_i) & j=1, c_{t-1}=\mathbf{0}, \text{ and } w_t=\symopen_i\\
    \psi^{-1}(\symopen_i) & c_{t-1} \text{ is } (j-1)\text{-top and } w_t = \symopen_i\\
    s_{t-1,j} & o.w.
  \end{cases}
\end{align}
Intuitively, the first case of Equation~\ref{eqn_stack_recurrence} implements pop, the second case implements push (if pushing to an otherwise empty stack) and the third case implement push (if pushing to a non-empty stack).
The fourth case specifies that slots that are not pushed to or popped are maintained from timestep $t-1$ to timestep $t$.

So far, our definitions of the LSTM memory have concerned the cell state, $c_t$.
However, the values of the gates $f_t, i_t, o_t$, as well as the new cell candidate $\tilde{c}_t$, are functions not of the previous cell state but the previous hidden state, $h_{t-1} = o_{t-1} \odot \text{tanh}(c_{t-1})$.
The dynamics of $h_t$ under $f_\theta$ we specify as follows, using the same notation $h_{t,j}$ to refer to the dimensions of $h_t$ that correspond to slack slot $s_{t,j}$ of $c_t$:
\begin{align}\label{eqn_stack_top}
  h_{t,j} = \begin{cases}
  \text{tanh}(\psi^{-1}(\symopen_i)) & c_t \text{ is $j$-top}\\
  &\text{and } s_{t,j} = \psi^{-1}(\symopen_i)\\
    \mathbf{0} \in \mathbb{R}^k & o.w.,
  \end{cases}
\end{align}
Intuitively, this means \textbf{only the top of the stack is stored in the hidden state}, and it's stored in whichever units correspond to the $j$-top slot in $c_t$.
Since $\psi^{-1}(\symopen_i)$ takes on values in $\{0,1\}$, $\text{tanh}(\psi^{-1}(\symopen_i))$ takes on values in $\{0, \text{tanh}(1)\}$.

We're now ready to formalize the first part of our proof in a lemma that encompasses how $f_\theta$ structures and manipulates its memory.
\begin{lemma}[stack correspondence] \label{appendix_lemma_lstm_stack_correspondence} %
  For all strings $w_{1:T}$ in \dyckkm, for all $t=1,...,T$,
  let $q_t$ be the state of $D_{m,k}$ after consuming prefix $w_{1:t}$.
  Let $c_t$ be the hidden state of the LSTM after consuming $w_{1:t}$.
  Then if $q_t=\not \in \{[\symend]\}$, 
  \begin{align} \label{eqn_lemma1_eqn1}
    \mathcal{Q}(s_{t,1},\dots,s_{t,m}) = q_t,
  \end{align}
  and letting $q_t = [\symopen_{i_1},\dots,\symopen_{i_{m'}}]$ for $m'\leq m$ without loss of generality, 
  \begin{align}\label{eqn_lemma1_eqn2}
  h_{t,j} = \begin{cases}
    \text{tanh}(\psi^{-1}(\symopen_{i_{m'}})) & j=m'\\
      \mathbf{0} & o.w.
  \end{cases}
  \end{align}
  where if $q_t=[]$, then $m'=0$ and $h_{t,j}=\mathbf{0}$ for all $j$.
\end{lemma}
Note again that the value of $\psi^{-1}$ is not specified by the lemma; whereas so far we've used $\psi^{-1}(\symopen_i)=e_i$, a one-hot encoding, in (\S~\ref{appendix_sec_efficient_generating}) we'll replace $\psi^{-1}$ while maintaining the lemma.

Intuitively, Equation~\ref{eqn_lemma1_eqn1} states that $f_\theta$ keeps an exact representation of the stack of unclosed open brackets so far, a statement solely concerning the cell state $c_t$.
However, all intermediate values in the LSTM equations are functions of the hidden state $h_t$, not the cell state directly; Equation~\ref{eqn_lemma1_eqn2} states how the top of the stack is represented in the hidden state, $h_t$.
We state these as one lemma since it is convenient to prove them by induction together.

\subsection{Proof of the stack correspondence lemma assuming stack slot dynamics}
We first prove Lemma~\ref{appendix_lemma_lstm_stack_correspondence} under the assumption that the stack slot dynamics given in Equation~\ref{eqn_stack_recurrence} hold.
We proceed by induction on the prefix length.

When the prefix length is $0$, the DFA is in state $[]$, and $s_{0,j}=\mathbf{0}$ for all $j$, so $\mathcal{Q}(s_{0,1},\dots,s_{0,m})=[]$, as required.
Further, we have $h_0=\mathbf{0}$, as required.
This completes the base case.

Assume for the sake of induction that Lemma~\ref{appendix_lemma_lstm_stack_correspondence} holds for strings up to length $t$.
Now consider a prefix $w_{1:t+1}$. 
From it, we can take the prefix $w_{1:t}$.
Running $f_\theta$, we have $s_{t,1},\dots,s_{t,m}$.
By the induction hypothesis, we have that $\mathcal{Q}(s_{t,1},\dots,s_{t,m}) = q_t = [\symopen_{i_1},\dots,\symopen_{i_{m'}}]$, (Equation~\ref{eqn_lemma1_eqn1}) as well as that
\begin{align*}
  h_{t,j} = \begin{cases}
    \text{tanh}(\psi^{-1}(\symopen_{i_{m'}})) & j=m'\\
      \mathbf{0} & o.w.
  \end{cases}
\end{align*}
by Equation~\ref{eqn_lemma1_eqn2}.
Now, the symbol $w_{t+1}$ can be one of $2k+1$ symbols: any of the $k$ open brackets $\symopen_i$, the $k$ close brackets $\symclose_i$, or $\symend$.
The proof now proceeds by each of these three cases.
\paragraph{Case $w_{t+1} = \symopen_i$:}
Without loss of generality, we let $q_t = [\symopen_{i_1},\dots,\symopen_{i_{m'}}]$ for $m'<m$. 
Strict inequality is guaranteed since if $m'=m$, then $\delta(q_t,\symopen_i)=r$, meaning $w_{1:t+1}$ cannot be a prefix of any string in \dyckkm, contradicting a premise of the lemma.
Through the inverse of the stack mapping Equation~\ref{eqn_stack_map}, we have 
\begin{align*}
  \mathcal{Q}^{-1}(q_t) &= s_{t,1},\dots,s_{t,m}\\
  &= \psi^{-1}(\symopen_{i_1}),\dots,\psi^{-1}(\symopen_{i_{m'}}),\mathbf{0}_{m'+1},\dots,\mathbf{0}_{m}\\
  &= e_{i_1},\dots,e_{i_{m'}},\mathbf{0}_{m'+1},\dots,\mathbf{0}_{m},
\end{align*}
the stack slots of $f_\theta$ at timestep $t$.
Since $s_{t,m'}=\psi^{-1}(\symopen_{i_j})$ and $s_{t,m'+1}=\mathbf{0}$, we have that $c_t$ is $m'$-top.
Thus, since $w_{t+1}=\symopen_i$, the second or third condition in Equation~\ref{eqn_stack_recurrence} (depending on whether $m'=0$) dictate that $s_{t+1,m'+1} = \psi^{-1}(\symopen_i)$ (where the $i$ is because $w_{t+1} = \symopen_i$).
All other stack slots fall under the condition $s_{t+1,j} = s_{t,j}$.

With this, we can reason about the DFA state encoded by $c_{t+1}$ at timestep $w_{1:t+1}$:
\begin{align*}
  \mathcal{Q}&(s_{t+1,1},\dots,s_{t+1,m}) \\
  &= \mathcal{Q}(\psi^{-1}(\symopen_{i_1}),\dots,\psi^{-1}(\symopen_{i_{m'}}),\psi^{-1}(\symopen_{i}),\mathbf{0},\dots,\mathbf{0})\\
  &= [\symopen_{i_1},\dots,\symopen_{i_{m'}},\symopen_i]\\
  &= \delta(q_t, \symopen_i)
\end{align*}
which is the DFA's state at timestep $t+1$, as required.

Finally, we reason about $h_{t+1,j}$.
We've just shown that  that $c_{t+1}$ is $(m'+1)$-top, and that $s_{t+1,m'+1}=\psi^{-1}(\symopen_i)$.
By Equation~\ref{eqn_stack_top}, we have that $h_{t+1,m'+1}= \text{tanh}(\psi^{-1}(\symopen_i))$, and $h_{t+1,j}=\mathbf{0}$ for all $j\not= m'+1$ as required, completing this case. %
\paragraph{Case $w_{t+1} = \symclose_i$:} 
We have $q_t=[\symopen_{i_1},\dots,\symopen_{i_{m'}}]$ for $0<m'$.
this inequality is strict since if $m'=0$, then $q_t=[]$, and $\delta(q_t,\symclose_i)=r$, meaning $w_{1:t+1}$ is not a prefix of any string in $D_{m,k}$.
Thus, we have that $c_t$ is $m'$-top.
Since $w_{t+1} = \symclose_i$, we have that $s_{t+1,m'} = \mathbf{0}$, and $s_{t+1,j}=s_{t,j}$ for all $j\not=m'$.
Thus, we have
\begin{align*}
  s_{t+1,1},\dots,s_{t+1,m} &= \psi^{-1}(\symopen_{i_1}),\dots,\psi^{-1}(\symopen_{i_{m'-1}})
  \\&,\mathbf{0},\dots,\mathbf{0}
\end{align*}
and thus the DFA state corresponding to the stack slots is:
\begin{align*}
\mathcal{Q}(s_{t+1,1}&,\dots,s_{t+1,m})\\
&=[\psi^{-1}(s_{t+1,1}),\dots,\psi^{-1}(s_{t+1,m'-1})\\
&\ \ \ \ \ \ , \mathbf{0},\dots,\mathbf{0}]\\
&=[\symopen_{i_1},\dots,\symopen_{m'-1}]
\end{align*}
Which is $\delta(q_t, \symclose_i)$, as required.
By the same reasoning as for case $w_{t+1}=\symopen_i$, since $c_{t+1}$ is $(m'-1)$-top and that $s_{t+1,m'-1}=\psi^{-1}(\symopen_{m'-1})$, we have that $h_{t+1,m'-1} = \text{tanh}(\psi^{-1}(\symopen_i))$ and $h_{t+1,j}=\mathbf{0}$ for all $j\not=m'-1$, as required.
This completes the case.

\paragraph{Case $w_{t+1} = \symend$:} 
If $w_{t+1}=\symend$, then by the definition of $\delta$, $\delta(q_t, \symend)\in\{[\symend], r\}$; hence the premises of Lemma~\ref{appendix_lemma_lstm_stack_correspondence} don't hold, so the lemma vacuously holds.

\paragraph{Summary}
In this section we've proved the stack correspondence lemma for our LSTM construction, assuming we can implement the dynamics of the model as specified in Equations~\ref{eqn_stack_recurrence}~\ref{eqn_stack_top}; we have yet to show that the LSTM update can implement the equations as promised.

\subsection{Implementation of stack state dynamics in LSTM equations assuming intermediate values}
In this subsection, we define how the dynamics defined in Equations~\ref{eqn_stack_recurrence},~\ref{eqn_stack_top} are implemented in the LSTM equations, referring to intermediate gate values, assuming such values can be reached with some setting of LSTM parameters.

As we discussed stack slots $s_{t,1},\dots,s_{t,m}$ of cell state $c_t$, where each $s_{t,j}\in\mathbb{R}^k$, we similarly discuss the gates in terms of slots, where e.g., $f_{t,j}\in[0,1]^k$ are the $k$ elementwise gates that apply to the values $s_{t,j}$ (and similarly for $i_{t,j}$, $o_{t,j}$, and the new cell candidate values $\tilde{c}_{t,j}\in[-1,1]^k$.)
With this, we re-write Equation~\ref{eqn_stack_recurrence} using the definition of $c_t$:
\begin{align} \label{eqn_cell_dynamic}
  &s_{t,j} = f_{t,j} \odot s_{t-1,j} + i_{t,j}\odot \psi^{-1}(w_t),
\end{align}
as well as Equation~\ref{eqn_stack_top}:
\begin{align}\label{eqn_hidden_dynamic}
  h_{t,j} = o_{t,j} \odot \text{tanh}(s_{t,j})
\end{align}
How  these fulfill Equations~\ref{eqn_stack_recurrence},~\ref{eqn_stack_top}?
We can set the gate values to do so as follows.
First, the forget gate is used to detect the condition under which slot $j$ should be erased:
\begin{align} \label{eqn_f_gate}
  f_{t,j} = \begin{cases}
    \mathbf{0}\in\mathbb{R}^k & c_{t-1} \text{ is $j$-top and } w_t = \symclose_i\\
    \mathbf{1}\in\mathbb{R}^k & o.w.
  \end{cases}
\end{align}
Second, the new cell candidate expression is used to give the \textit{option} of writing a new open bracket, $\psi^{-1}(w_t)$ to \textit{any} stack slot:
\begin{align} \label{eqn_cell_candidate}
  \tilde{c}_{t,j} = \begin{cases}
    \psi^{-1}(\symopen_i) & w_t = \symopen_i\\
    \mathbf{0}\in\mathbb{R}^k & o.w.
  \end{cases}
\end{align}
Third, the input gate determines which stack slot is the first non-empty slot in the sequence, and only allows the information of the new open bracket to be written to that slot:
\begin{align} \label{eqn_i_gate}
  i_{t,j} = \begin{cases}
    \mathbf{1} & c_{t-1} =\mathbf{0},j=1, \text{ and } w_t =\symopen_i\\
    \mathbf{1} & c_{t-1} \text{ is $(j-1)$-top, and } w_t =\symopen_i\\
    \mathbf{0}\in\mathbb{R}^k & o.w.
  \end{cases}
\end{align}
Finally, the output gate is set to identify the top of the stack and not let any other non-$\mathbf{0}$ (i.e., all stack slots below the top) through:
\begin{align} \label{eqn_o_gate}
  o_{t,j} = \begin{cases}
    \mathbf{0} & c_t \text{ is $j'$-top for $j'> j$}\\
    \mathbf{1}\in\mathbb{R}^k & o.w.
  \end{cases}
\end{align}
Many of these conditions refer to values, like $c_t$ and $c_{t-1}$, not available when computing the gates (as only $h_{t-1}$ and $x_t$  are available); it is convenient to refer to these conditions and later show how they are implemented with the available values.

\subsection{Proof of stack state dynamics given gate and new cell candidate values}
In this section, we prove that Equations~\ref{eqn_stack_recurrence},~\ref{eqn_stack_top} are implemented by the LSTM $f_\theta$ given the gate and new cell candidate values defined in the previous subsection.

We start with Equation~\ref{eqn_stack_recurrence}, the definition of the stack slot dynamics, proceeding by each of the four cases for defining $s_{t,j}$ in Equation~\ref{eqn_stack_recurrence}.

\paragraph{Case 1 (popping $s_{t-1,j}$)} In this case, the condition is that $c_{t-1}$ is $j$-top, and $w_t = \symclose_i$.
By Equation~\ref{eqn_f_gate}, we have the forget gate value $f_{j,t}=\mathbf{0}$.
By Equation~\ref{eqn_i_gate}, we have the input gate value $i_{j,t} = \mathbf{0}$.
Finally, by Equation~\ref{eqn_cell_candidate}, we have $\tilde{c}_{t,j} = \mathbf{0}$.
Plugging these values into the LSTM cell expression in Equation~\ref{eqn_cell_dynamic}, we get
\begin{align*}
  s_{t,j} &= \mathbf{0} \cdot s_{t-1,j} + \mathbf{0} \cdot \mathbf{0}\\
  &= \mathbf{0},
\end{align*}
as required.

\paragraph{Case 2 (pushing to $s_{t,1}$)} In this case, the condition is that $c_{t-1}=\mathbf{0}$, that is, the stack is empty, $j=1$, that is, we'll write to the first stack slot, and $w_t=\symopen_i$.
By Equation~\ref{eqn_i_gate}, we have $i_{t,1}=\mathbf{1}$.
By Equation~\ref{eqn_cell_candidate}, we have $\tilde{c}_{t,j} = \psi^{-1}(\symopen_i)$.
Further, we have $s_{t-1,j} = \mathbf{0}$ for all $j$, since $c_{t-1}=\mathbf{0}$.
Plugging these valeus into the LSTM cell expression in Equation~\ref{eqn_cell_dynamic}, we get:
\begin{align*}
  s_{t,1} &= f_{1,j} \cdot \mathbf{0} + \mathbf{1} \cdot \psi^{-1}(\symopen_i)\\
  &= \psi^{-1}(\symopen_i),
\end{align*}
as required.

\paragraph{Case 3 (pushing to $s_{t,j>1}$)} In this case, the condition is that $c_{t-1}$ is $(j-1)$-top or $c_{t-1}=\mathbf{0}$, and $w_t=\symopen_i$.
By Equation~\ref{eqn_f_gate}, we have the forget gate value $f_{j,t}=\mathbf{1}$, since Case 1 does not hold.
By Equation~\ref{eqn_i_gate}, we have the input gate value $i_{j,t} = \mathbf{1}$.
Finally, by Equation~\ref{eqn_cell_candidate}, we have $\tilde{c}_{t,j} = \psi^{-1}(\symopen_i)$.
Plugging these values into the LSTM cell expression in Equation~\ref{eqn_cell_dynamic}, we get
\begin{align*}
  s_{t,j} &= \mathbf{1} \cdot s_{t-1,j} + \mathbf{1} \cdot \psi^{-1}(\symopen_i)\\
  &= s_{t-1,j} + \psi^{-1}(\symopen_i)\\
  &= \psi^{-1}(\symopen_i),
\end{align*}
where the last equality hold because of the condition that $c_{t-1}$ is $(j-1)$-top, which implies $s_{t-1,j} = \mathbf{0}$.
\paragraph{Case 3 (maintaining $s_{t-1,j}$)} This case catches all conditions under which Cases 1 and 2 do not hold.
By Equation~\ref{eqn_f_gate}, we have the forget gate value $f_{j,t}=\mathbf{1}$, since Case 1 does not hold.
By Equation~\ref{eqn_i_gate}, we have the input gate value $i_{j,t} = \mathbf{0}$, since Case 2 does not hold.
Finally, by Equation~\ref{eqn_cell_candidate}, we have $\tilde{c}_{t,j} = \mathbf{0}$.
Plugging these values into the LSTM cell expression in Equation~\ref{eqn_cell_dynamic}, we get
\begin{align*}
  s_{t,j} &= \mathbf{1} \cdot s_{t-1,j} + \mathbf{0} \cdot \mathbf{0}\\
  &= s_{t-1,j},
\end{align*}
as required.

\paragraph{Proving Equation~\ref{eqn_stack_top} is implemented}
Next, we prove that Equation~\ref{eqn_stack_top} is implemented by the LSTM $f_\theta$, given that Equation~\ref{eqn_stack_recurrence} is.
First, we must show $h_{t,j}$ is equal to $\text{tanh}(e_i)$ if $c_t$ is $j$-top. If $c_t$ is $j$-top, it must be that $s_{t,j}=e_i$ for some $i$.
Since $c_t$ is $j$-top, we have $o_{t,j}=\mathbf{1}$, and $s_{t,j}=e_{i}$.
Plugging these values into the hidden state expression in Equation~\ref{eqn_hidden_dynamic}, we get:
\begin{align*}
  h_{t,j} &= o_{t,j}\odot \text{tanh}(s_{t,j})\\
  &= \mathbf{1}\odot \text{tanh}(e_{i})\\
  &= \text{tanh}(e_i),
\end{align*}
as required.

If that condition does not hold, we must have that $c_{t}$ not $j$-top.
thus, $c_{t}$ is either $j$-top for some $j'<j$ or $j'>j$, or $c_{t}=\mathbf{0}$.

If $c_{t}$ is $j'$-top for $j'<j$, then we have that $s_{t,j} = \mathbf{0}$, by the definition of $j$-top.
In this case, we have
\begin{align*}
  h_{t,j} &= o_{t,j}\odot\text{tanh}(s_{t,j})\\
  &= \mathbf{1} \odot \text{tanh}(\mathbf{0})\\
  &= \mathbf{0},
\end{align*}
as required.

If $c_{t}$ is $j'$-top for $j'>j$, then we have that $o_{t,j} = \mathbf{0}$, by Equation~\ref{eqn_o_gate}.
In this case, we have
\begin{align*}
  h_{t,j} &= o_{t,j}\odot \text{tanh}(s_{t,j})\\
  &= \mathbf{0} \odot s_{t,j}\\
  &= \mathbf{0},
\end{align*}
as required.

Finally, if $c_t = \mathbf{0}$, meaning it is not $j$-top for any $j$, then
\begin{align*}
  h_{t,j} &= o_{t,j}\odot \text{tanh}(\mathbf{0})\\
  &= \mathbf{0},
\end{align*}
as required.

\paragraph{Summary}
So far, we've proved the stack correspondence lemma's induction step for LSTMs assuming that we can provide parameters of the network such that the values assumed in Equations~\ref{eqn_f_gate},~\ref{eqn_cell_candidate},~\ref{eqn_i_gate},~\ref{eqn_o_gate}, that is, the gates and new cell candidate, are achieved.
We have yet to provide the settings of parameters to do so.

\subsection{Construction and proof of LSTM parameters providing gate and new cell candidate values} \label{section_lstm_params_appendix}
We now come to the final portion of the proof of the inductive step in Lemma~\ref{appendix_lemma_lstm_stack_correspondence}.
So far, we've proven the induction step conditioned on the existence of parameters that implement the intermediate gate values $f_{t,j}$, $i_{t,j}$, $\tilde{c}_{t,j}$, $o_{t,j}$.
Now we construct said parameters, and prove their correctness assuming the induction hypothesis.
We'll describe all matrices $W$, that is, ($W_i, W_f, W_0, W_{\tilde{c}}$), as block matrices, with the block structure specified by the stack slot structure.
Precisely, we'll refer to $W_{j,j'}$ as rows $(j-1)k+1$ through $(j-1)k+k$, intersected with columns $(j'-1)k+1$ through $(j'-1)k+k$.
This can be interpreted as the block that specifies how the values of slot $j'$ affect the intermediate values constructed for slot $j$.
We'll describe all matrices $U$ as block-matrices through rows alone, that is, $U_{j}$ referring to rows $(j-1)k'+1$ through $(j-1)k+k$.
We'll describe biases $b$ through contiguous chunks $b_j$, similarly.

\paragraph{Embedding parameters}
By definition, $x_t = Ew_t$, where each $w_t\in\Sigma$ is overloaded to be a one-hot vector picking out a row of embedding matrix $E\in\mathbb{R}^{d\times |\Sigma|}$.
We define $E=I^{2k\times 2k}$, the identity matrix.
Note that $\{Ew\}_{w\in\Sigma}$ is a set of mutually orthogonal vectors, with $\|Ew\|_2=1$.

\paragraph{Parameters implementing $o_{t,j}$ as Equation~\ref{eqn_o_gate}}
The equation defining $o_{t,j}$ is as follows:
\begin{align}
  o_{t,j} &= [\sigma(W_{o}h_{t-1} + U_ox_{t} + b_o)]_j\\
  &= \sigma([W_{o}h_{t-1}]_j + [U_ox_{t}]_j + [b_o]_j)
\end{align}
where the subscript $j$ on the RHS indicates that we're picking the units corresponding to stack slot $j$.
Let $\gamma = \text{tanh}(1)\approx 0.762$ for ease of notation.
We define $W_{o,j,j'}$ as follows:
\begin{align}
  W_{o,j,j'} = 
    \begin{cases}
      \mathbf{-\rscale} \in \mathbb{R}^{k\times k} & j'\in\{j,j+1\}\\
      \mathbf{-2\rscale} \in \mathbb{R}^{k\times k} & j'>j+1\\
      \mathbf{0} \in \mathbb{R}^{k\times k} & j'<j
    \end{cases},
\end{align}
where $\rscale \in \mathbb{R}$ is a scaling factor to be defined.
This gives the intermediate value:
\begin{align} \label{eqn_h_dotprod_output_gate}
  [W_{o}h_{t-1}]_j = \begin{cases}
    \mathbf{-\rscale\gamma}  & \exists_{j'\in\{j,j+1\},i} :\\&h_{t-1,j'} = \text{tanh}(\psi^{-1}(\symopen_i))\\
    \mathbf{-2\rscale\gamma}  & \exists_{j'>j+1,i} :\\&h_{t-1,j'} = \text{tanh}(\psi^{-1}(\symopen_i))\\
    \mathbf{0} \in \mathbb{R}^k & o.w.
  \end{cases}
\end{align}
Note that Equation~\ref{eqn_h_dotprod_output_gate} relies on the encoding-validity property, in particular that $\psi^{-1}(\symopen_i)\in\{0,1\}^k$ and $\sum_{\ell=1}^{k}\psi^{-1}(\symopen_i)=1$ to ensure that, in the first condition, 
\begin{align*}
  -\lambda\mathbf{1}^\top h_{t-1,j} &= -\lambda\mathbf{1}^\top\text{tanh}(s_{t-1,j})\\
  &= -\lambda\mathbf{1}^\top\text{tanh}(\psi^{-1}(\symopen_i))\\
  &= -\lambda\sum_{\ell=1}^{k}\text{tanh}(\psi^{-1})_\ell\\
  &= -\lambda\text{tanh}(1)\\
  &= -\lambda\gamma.
\end{align*}
and likewise for the second condition.
The third condition relies on the fact that $\psi^{-1}(\varnothing)=\mathbf{0}$.

We define $U_{o,j}$  as follows:
\begin{align}
  U_{o,j} = \begin{bmatrix}
   \rscale\gamma \sum_{i} x_{\symclose_i}^\top\\
   \dots\\
  \end{bmatrix}
\end{align}
that is, each row of $U_{f}$ is equal to the sum, over all close brackets, of the embedding of that close bracket (transposed.)
Since all $x_{\symclose_i}$ are orthogonal and unit norm (by the definition of $E$), this gives the following intermediate value:
\begin{align}
  [U_ox_t]_j = \begin{cases}
    \mathbf{\rscale}\gamma \in \mathbb{R}^k & \exists_i : w_t = \symopen_i\\
    \mathbf{0} & o.w.
  \end{cases}
\end{align}
Finally, we specify the bias term as
\begin{align}
  b_{o,j} = \begin{bmatrix}
    0.5\rscale\gamma\\
    \dots\\
  \end{bmatrix},
\end{align}
that is, all units of the bias are equal to the same value.

As we stated earlier, Equation~\ref{eqn_o_gate} refers to conditions on $c_t$, which do not show up in the computation of $o_t$; instead, $o_t$ is a function of $h_{t-1}$, and $x_t$.
As such, we re-write Equation~\ref{eqn_o_gate} in terms of these values, in particular splitting the condition on $c_t$ being $j'$-top for $j'>j$ into two separate conditions,
\begin{align} \label{eqn_o_gate2}
  o_{t,j} = \begin{cases}
    \mathbf{0} & \exists_{j'>j+1,i}: h_{t-1,j'} = \tanh(\psi^{-1}(\symopen_i))\\
    \mathbf{0} & \exists_{j'\in\{j,j+1\},i}: h_{t-1,j'} = \tanh(\psi^{-1}(\symopen_i))\\
    &\text{and } w_{t}\not=\symclose_i\\
    \mathbf{1} & o.w.
  \end{cases}
\end{align}
Recall that $h_{t,j'}=\text{tanh}(\psi^{-1}(\symopen_i))$ indicates that $c_{t-1}$ is $j'$-top, by the induction hypothesis.
With our equation above we outline two different conditions that we'll show together are necessary and sufficient to guarantee that the top slot is located at some $j'>j$. Under each condition, we set the output gate to $0$ because of that (since if the top slot is at $j'>j$, then slot $j$ is not equal to $0$, and must not be let through the gate.)

While we cannot condition directly on the value of $c_t$, we know that $c_t$ is $j'$-top for $j'>j$ under two conditions on $h_{t-1}$, and thus on $c_{t-1}$.
First, $c_{t-1}$ can be $(j+2)$-top or greater; since only one element can be popped at once, $c_t$ can be no less than $(j+1)$-top.
Second, $c_{t-1}$ can be $j$-top or $(j+1)$-top, and the input $w_t$ is not some $\symclose_i$; that is, it does not pop the top element off of the stack.
Because $w_t$ is not $\symclose_i$, it must be an open bracket, $\symopen_i$, pushing to the stack.
Hence, under these conditions, $c_t$ must be at least $(j+1)$-top.
If neither of these conditions hold, then $c_t$ cannot be $j'$-top for $j'>j$, since if $c_{t-1}$ is $j$-top or $(j-1)$-top and $w_t = \symclose_i$, then $c_t$ must be at most $j$-top.
And if $c_{t-1}$ is $j'$-top for $j'<j$, then it is impossible since only one bracket can be pushed at a time.
Thus Equation~\ref{eqn_o_gate2} is equivalent to Equation~\ref{eqn_o_gate}, as required.

We now prove that parameters $(W_o, U_o, b_o)$ implement Equation~\ref{eqn_o_gate}, by implementing Equation~\ref{eqn_o_gate2}.
In the first condition, $h_{t-1,j'}= \psi^{-1}(\symopen_i)$ for some $j'>j+1$ and $i$; we want to show that $o_{t,j}=\mathbf{0}$.
Based on our construction of parameters, we have that $o_{t,j}$ is upper-bounded by the following, when $w_t=\symclose_i$:
\begin{align*}
  o_{t,j} = &\sigma([W_{o}h_{t-1}]_j + [U_ox_{t}]_j + [b_o]_j)\\
  &\leq \sigma(-2\rscale \gamma +\rscale\gamma + .5\rscale\gamma)\\
  &= \sigma(-0.5\rscale\gamma)\\
  &= \mathbf{0}
\end{align*}
where the last equality holds under our finite precision $\sigma$ by setting the scaling factor $\rscale$ such that $-0.5\rscale\gamma < -\beta$, that is $\rscale > \frac{2\beta}{\gamma}$.

In the next condition, we have that $h_{t-1,j'}=e_i$for $j'\in\{j,j+1\}$, and $i$, and $x_t \not = \symclose_i$.
In this case, we can write out the value of $o_{t,j}$ as follows:
\begin{align*}
  o_{t,j} &= \sigma(-\rscale \gamma + 0 + .5\rscale\gamma)\\
  &= \sigma(-0.5\rscale\gamma)\\
  &= \mathbf{0},
\end{align*}
as required.

Finally, if neither of those conditions hold, we have that $h_{t-1,j'} \not= \psi^{-1}(\symopen_i)$ for any $j'\geq j$, and so $h_{t-1,j'}=\text{tanh}(\mathbf{0})=\mathbf{0}$.
In this case, we can lower-bound the value of $o_{t,j}$:
\begin{align*}
  o_{t,j} &\leq \sigma(0 + 0 + 0.5\rscale\gamma)\\
  &= \mathbf{1},
\end{align*}
as required.
This completes the proof of the output gate.

\paragraph{Parameters implementing $f_{t,j}$ as Equation~\ref{eqn_f_gate}}
The equation defining $f_{t,j}$ is as follows:
\begin{align}
  f_{t,j} &= [\sigma(W_{f}h_{t-1} + U_fx_{t} + b_f)]_j\\
  &= \sigma([W_{f}h_{t-1}]_j + [U_fx_{t}]_j + [b_f]_j)
\end{align}
We define $W_{f,j,j'}$ as follows:
\begin{align}
  W_{f,j,j'} = 
    \begin{cases}
      \mathbf{-\rscale} \in \mathbb{R}^{k\times k} & j=j'\\
      \mathbf{0} \in \mathbb{R}^{k\times k} & o.w.
    \end{cases}
\end{align}
which gives the intermediate value:
\begin{align}\label{eqn_h_dotprod_forget_gate}
  [W_{f}h_{t-1}]_j = \begin{cases}
    \mathbf{-\rscale\gamma} \in \mathbb{R}^k & \exists_i : h_{t-1,j} \\
    & \ \ \ \ \ = \text{tanh}(\psi^{-1}(\symopen_i))\\
    \mathbf{0} \in \mathbb{R}^k & o.w.,
  \end{cases}
\end{align}
  again relying on the embedding-validity property.

We define $U_{f,j}$  as follows:
\begin{align}
  U_{f,j} = \begin{bmatrix}
   - \rscale\gamma\sum_{i} x_{\symclose_i}^\top\\
   \dots\\
  \end{bmatrix}
\end{align}
that is, each row of $U_{f}$ is equal to the negative of the sum, over all close brackets, of the embedding of that close bracket (transposed.)
Since all $x_{\symclose_i}$ are orthogonal and unit norm, this gives the following intermediate value:
\begin{align}
  [U_fx_t]_j = \begin{cases}
    \mathbf{-\rscale\gamma} \in \mathbb{R} & \exists_i : w_t = \symclose_i\\
    \mathbf{0} & o.w.
  \end{cases}
\end{align}
Finally, we specify the bias term as
\begin{align}
  b_{f,j} = \begin{bmatrix}
    1.5\rscale\gamma \\
    \dots\\
  \end{bmatrix},
\end{align}
that is, all units of the bias are equal to the same value.

We now prove that, as given, the parameters $(W_f, U_f, b_f)$ implement Equation~\ref{eqn_f_gate} assuming the induction hypothesis of Lemma~\ref{appendix_lemma_lstm_stack_correspondence}.

Equation~\ref{eqn_f_gate} has two cases.
In the first case, $c_{t-1}$ is $j$-top and $w_t=\symclose_i$, and we need to prove $f_{t,j}=\mathbf{0}$.
From the induction hypothesis, this means that $h_{t-1,j} = \text{tanh}(\psi^{-1}(\symopen_i))$, and all other slots $h_{t-1,j'}=\mathbf{0}$.
The value of the gate is as follows:
\begin{align*}
  f_{t,j} &= \sigma([W_{f}h_{t-1}]_j + [U_fx_{t}]_j + [b_f]_j)\\
  &= \sigma(-\rscale\gamma - \rscale\gamma + 1.5\rscale\gamma)\\
  &= \sigma(-0.5\rscale\gamma)\\
  &= \mathbf{0}
\end{align*}
where the last equality holds by setting $\rscale$ as already stated: $\rscale>\frac{2\beta}{\gamma}$.

The second case is whenever the conditions of the first case don't hold, in which case we must prove $f_{t,j}=\mathbf{1}$.
Thus, we either have $c_{t-1}$ \text{not} $j$-top, in which case,
\begin{align*}
  f_{t,j} &\geq \sigma(0 - \rscale\gamma + 1.5\rscale\gamma)\\
  &= \sigma(0.5\rscale\gamma)\\
  &= \mathbf{1},
\end{align*}
or we have $w_t\not=\symclose_i$, in which case,
\begin{align*}
  f_{t,j} &\geq \sigma(-\rscale\gamma -0 + 1.5\rscale\gamma)\\
  &= \sigma(0.5\rscale\gamma)\\
  &= \mathbf{1},
\end{align*}
as required.

\paragraph{Parameters implementing $i_{t,j}$ as Equation~\ref{eqn_i_gate}}
The equation defining $i_{t,j}$ is as follows:
\begin{align}
  i_{t,j} &= [\sigma(W_{i}h_{t-1} + U_ix_{t} + b_i)]_j\\
  &= \sigma([W_{i}h_{t-1}]_j + [U_ix_{t}]_j + [b_i]_j)
\end{align}
We define $W_{i,j,j'}$ as follows:
\begin{align}
  W_{i,j,j'} = 
    \begin{cases}
      \mathbf{-\rscale} \in \mathbb{R}^{k\times k} & j=1\\
      \mathbf{\rscale} \in \mathbb{R}^{k\times k} & j=j'+1\\
      \mathbf{0} \in \mathbb{R}^{k\times k} & o.w.
    \end{cases}
\end{align}
which gives the intermediate value:
\begin{align}\label{eqn_h_dotprod_input_gate}
  [W_{i}h_{t-1}]_j = \begin{cases}
    \mathbf{-\rscale\gamma} \in \mathbb{R}^k & j=1, \exists_{i,j'} :\\&h_{t-1,j'} = \text{tanh}(e_i)\\
    \mathbf{\rscale\gamma} \in \mathbb{R}^k & \exists_i : h_{t-1,j-1} \\
    &= \text{tanh}(e_i)\\
    \mathbf{0} \in \mathbb{R}^k & o.w.
  \end{cases}
\end{align}
once again relying only on the encoding-validity property.

We define $U_{i,j}$  as follows:
\begin{align}
  U_{i,j} = \begin{bmatrix}
   - \rscale\gamma\sum_{i} x_{\symopen_i}^\top\\
   \dots\\
  \end{bmatrix}
\end{align}
Since all $x_{\symclose_i}$ are orthogonal and unit norm, this gives the following intermediate value:
\begin{align}
  [U_ix_t]_j = \begin{cases}
    \mathbf{-\rscale\gamma} \in \mathbb{R}^k & \exists_i : w_t = \symopen_i\\
    \mathbf{0} & o.w.
  \end{cases}
\end{align}
Finally, we specify the bias term as
\begin{align}
  b_{i,j} = \begin{cases}
    -.5\rscale\gamma\in\mathbb{R}^k & j=1 \\
    -1.5\rscale\gamma\in\mathbb{R}^k & o.w.\\
  \end{cases}
\end{align}
We now prove that the parameters $(W_i, U_i, b_i)$, when plugged into $f_\theta$, implement Equation~\ref{eqn_i_gate} assuming the induction hypothesis of Lemma~\ref{appendix_lemma_lstm_stack_correspondence}.
Equation~\ref{eqn_i_gate} has three cases.

The condition of the first case is that $c_{t-1}=\mathbf{0}$, $j=1$, and $w_t=\symopen_i$; we need to show that $i_{t,1}=\mathbf{1}$.
Since $c_{t-1}=\mathbf{0}$, we have that $h_{t-1}=\mathbf{0}$.
We can compute $i_{t,1}$ as follows:
\begin{align*}
  i_{t,1} &= \sigma([W_{i}h_{t-1}]_j + [U_ix_{t}]_j + [b_i]_j)\\
  &= \sigma(0 + \rscale\gamma - 0.5\rscale\gamma)\\
  &= \sigma(0.5\rscale\gamma)\\
  &= \mathbf{1}.
\end{align*}
note the bias term $-0.5\rscale\gamma$ is only set that way for $j=1$.

The condition of the second case is that $c_{t-1}$ is $(j-1)$-top, and $w_t = \symopen_i$; we need to show that $i_{t,j} = \mathbf{1}$.
From the induction hypothesis, we have that $h_{t-1,j-1} = \psi^{-1}(\symopen_{i'})$ for some $i'$, and $h_{t-1}$ is $0$ elsewhere.
We calculate $i_{t,j}$ as follows,
\begin{align*}
  i_{t,j>1} &= \sigma([W_{i}h_{t-1}]_j + [U_ix_{t}]_j + [b_i]_j)\\
  &= \sigma(\rscale\gamma + \rscale\gamma - 1.5\rscale\gamma)\\
  &= \sigma(0.5\rscale\gamma)\\
  &= \mathbf{1},
\end{align*}
where we do not have to consider the case $i_{t,1}$ since $c_{t-1}$ cannot be $0$-top, so this case cannot hold.

Finally, in the third case, none of the above conditions hold, and we need to prove $i_{t,j}=\mathbf{0}$.
There are a number of possibilities here to enumerate.
First, let $c_{t-1}=\mathbf{0}$ and $j=1$, but $w_t\not =\symopen_i$.
Then,
\begin{align*}
  i_{t,1} & = \sigma(0 + 0 - 0.5\rscale\gamma)\\
  &= \sigma(-0.5\rscale\gamma)\\
  &= \mathbf{0},
\end{align*}
Next, we let $c_{t-1}=\mathbf{0}$ and $w_t=\symopen_i$, but $j>1$.
Then,
\begin{align*}
  i_{t,1} & = \sigma(0 + \rscale\gamma - 1.5\rscale\gamma)\\
  &= \sigma(-0.5\rscale\gamma)\\
  &= \mathbf{0},
\end{align*}
Next, we let $j=1$, $w_t=\symopen_i$, but $c_{t-1}\not = \mathbf{0}$.
Thus $c_{t-1}$ is $j'$-top for some $j'$.

If the second case of Equation~\ref{eqn_i_gate} doesn't hold, then either $c_{t-1}$ is not $(j-1)$-top, or $w_t\not = \symopen_i$; we need to prove $i_{t,j}=\mathbf{0}$.
If $c_{t-1}$ is not $(j-1)$-top and $j>1$, then %
we can upper-bound the value of $i_{t,j}$ as follows:
\begin{align*}
  i_{t,j>1} &\leq \sigma(0 + \rscale\gamma -1.5\rscale\gamma)\\
  &= \sigma(-0.5\rscale\gamma)\\
  &=\mathbf{0}.
\end{align*}
If $c_{t-1}$ is $(j-1)$-top but $w_t\not=\symopen_i$, then the value of $i_{t,j}$ is as follows,
\begin{align*}
  i_{t,j>1} &= \sigma(\rscale\gamma + 0 -1.5\rscale\gamma)\\
  &= \sigma(-0.5\rscale\gamma)\\
  &=\mathbf{0},
\end{align*}
as required.

\paragraph{Parameters implementing $\tilde{c}_{t,j}$ as Equation~\ref{eqn_cell_candidate}}

The equation defining $\tilde{c}_{t,j}$ is as follows:
\begin{align}
  \tilde{c_{t,j}} &= [\text{tanh}(W_{\tilde{c}}h_{t-1} + U_{\tilde{c}}x_t = b_{\tilde{c}})]_j\\
  &= \text{tanh}([W_{\tilde{c}}h_{t-1}]_j + [U_{\tilde{c}}x_t]_j = [b_{\tilde{c}}]_j)
\end{align}
We define $W_{\tilde{c}}=\mathbf{0}$, which leads to the intermediate values:
\begin{align}
  [W_{\tilde{c}}h_{t-1}]_j = \mathbf{0}
\end{align}
We define $U_{\tilde{c}}$ as follows:
\begin{align}
  U_{\tilde{c},j} = \begin{bmatrix}
    \rscale x_{\symopen_1}^\top\\
    \rscale x_{\symopen_2}^\top\\
   \dots\\
    \rscale x_{\symopen_k}^\top
  \end{bmatrix}
\end{align}
which leads to the intermediate values:
\begin{align}
  [U_{\tilde{c}}x_t]_j = \begin{cases}
    \rscale e_i & w_t = \symopen_i\\
    \mathbf{0} & o.w.
  \end{cases}
\end{align}
We define $b_{\tilde{c}} = \mathbf{0}$.

We now prove that the parameters $W_{\tilde{c}}, U_{\tilde{c}}, b_{\tilde{c}}$ implement Equation~\ref{eqn_cell_candidate}.
There are two cases.

In the first case, $w_t=\symopen_i$, and need to show $\tilde{c}_{t,j}=\psi^{-1}(\symopen_i)$.
In this case, we have
\begin{align*}
  \tilde{c}_{t,j} &= \text{tanh}([W_{\tilde{c}}h_{t-1}]_j + [U_{\tilde{c}}x_t]_j = [b_{\tilde{c}}]_j)\\
  &= \text{tanh}(\mathbf{0} + \rscale e_i + 0)\\
  &= e_i,\\
  &= \psi^{-1}(\symopen_i)
\end{align*}
where the last equality holds by setting $\rscale>\beta$.\footnote{Which we've already required earlier by letting $\lambda>\frac{2\beta}{\gamma}$.}
In the second case, we have $w_t\not=\symclose_i$, and need to show $\tilde{c}_{t,j} = \mathbf{0}$.
In this case, we have
\begin{align*}
  \tilde{c}_{t,j} &= \text{tanh}([W_{\tilde{c}}h_{t-1}]_j + [U_{\tilde{c}}x_t]_j = [b_{\tilde{c}}]_j)\\
  &= \text{tanh}(\mathbf{0} + \mathbf{0} + \mathbf{0})\\
  &= \mathbf{0}
\end{align*}
as required.

\paragraph{Summary}
We've now specified all parameters of the LSTM and completed the induction step in our proof of the stack correspondence lemma, Lemma~\ref{appendix_lemma_lstm_stack_correspondence}; this completes the proof of the lemma.

\section{Proving generation in $O(mk)$ hidden units} \label{appendix_sec_generating}
For both the Simple RNN and the LSTM, we've proven \textit{stack correspondence lemmas}, Lemmas~\ref{appendix_lemma_simple_rnn_stack_correspondence},~\ref{appendix_lemma_lstm_stack_correspondence}.
These guarantee that for all prefixes of strings in \dyckkm, we can rely on properties of the hidden states of each model to be perfectly informative about the state of the DFA stack.

Given those lemmas, this section describes the proof of the following results, for the Simple RNN:
\thmsimplernntwomk* \label{appendix_thm_simple_rnn_2mk}
and for the LSTM, noting that the condition that $W_{\tilde{c}}=\mathbf{0}$ holds from our proof of the stack correspondence lemma.
\thmlstmmk* \label{appendix_thm_lstm_mk}

The two constructions differ in where they make information accessible; we describe those differences here and then provide a general proof that is agnostic to which construction is used.

\subsection{Softmax distribution parameters}
Recall that we must show that the $\epsilon$-truncated support of $f_\theta$ (for both Simple RNN and LSTM) is the same set as \dyckkm.
The token-level probability distribution conditioned on the history is given as follows:
\begin{align}
  w_t \sim \text{softmax}(Vh_{t-1} + b),
\end{align}
and we denote this distribution $p_{f_\theta}$.
We first specify $V$ by row; a single row exists for each of the $2k+1$ words in the vocabulary.
We let $v_{w,j}$ refer to the row of word $w$, and the $k$ columns that will participate in the dot product with the rows of stack slot $j$ in $h_{t-1}$.

We start with the only difference in the softmax matrix between the Simple RNN and the LSTM.
For the Simple RNN:
\begin{align}
  &v_{\symclose_i,1} = \pscale e_i\\
  &v_{\symclose_i,j>1} = \mathbf{0},
\end{align}
because the top of the stack is guaranteed to be in the first stack slot, and where $\pscale$ is a positive scaling constant we'll define later.
And for the LSTM,
\begin{align}
  &v_{\symclose_i,j} = \pscale e_i
\end{align}
for all $j$, since exactly one slot is guaranteed to be non-empty (if the stack is non-empty) and it could be at any of the $m$ slots.
The rest of the softmax construction is common between the Simple RNN and the LSTM:
\begin{align}
  &v_{\symopen_i,j<m} = \mathbf{0}\in\mathbb{R}^k\\
  &v_{\symopen_i,m} =  \mathbf{-1}\pscale\in\mathbb{R}^k\\
  &v_{\symend,j,1} =  \mathbf{-1}\pscale\in\mathbb{R}^k
\end{align}
Likewise, the bias terms:
\begin{align}
  &b_{v, \symclose_i} = -0.5\pscale\gamma\\
  &b_{v, \symopen_i} = 0.5\pscale \gamma\\
  &b_{v, \symend} = 0.5\pscale \gamma
\end{align}

So that we can swap out $V$ when we swap out $\psi$ for a more efficient encoding, we state here a properties of $V$ we rely on (once $\psi$ is fixed):
\begin{defn}[softmax validity property]\label{appendix_def_softmax_validity}
  Given encoding $\psi: \{\symopen_i\}_{i\in[k]}\rightarrow \mathbb{R}^n$, a softmax matrix $V$ obeys the softmax validity property relative to $\psi$ if
  \begin{align}
    \forall_{i\in[k]}, v_{\symclose_i,u}^\top \psi^{-1}(\symopen_j) \begin{cases}
      =\pscale & i = j\\
      \leq 0 & i \not = j,
    \end{cases}
  \end{align}
  where $u$ specifies the stack slot where the construction (Simple RNN or LSTM) stores the element at the top of the stack; for the Simple RNN, $u=1$, for the LSTM, $u$ is such that $c_t$ is $u$-top, as defined.
  This ensures that the softmax matrix correctly distinguishes between the symbol, $\symopen_i$, that is encoded in the top of the stack, from any other symbol.
  Further, we rely on
  \begin{align}
    \forall_{i,j\in[k]}, v_{\symopen_i,m}^\top \psi^{-1}(\symopen_j) = -\pscale\\
    \forall_{i,j\in[k]}, v_{\symopen_i,m'<m}^\top \psi^{-1}(\symopen_j) = 0
  \end{align}
  to detect whether the stack is full, and
  \begin{align}
    \forall_{i,j\in[k]}, \forall_{z\in[m]},v_{\symend,z}^\top \psi^{-1}(\symopen_j) = -\pscale
  \end{align}
  to detect if the stack is not empty.
\end{defn}
The softmax matrix we've provided obeys the softmax validity property with respect to $\psi$ since $\pscale e_i^\top e_i = \pscale$, and $\pscale e_i^\top e_j=0$, $j\not=i$.
Further, $-\pscale\mathbf{1}^\top e_i = -\pscale$ for all $i$.

We're now prepared to state the final lemma in our proof of Theorems~\hyperref[appendix_thm_simple_rnn_2mk]{\ref*{thm_simplernn2mk}},~\hyperref[appendix_thm_lstm_mk]{\ref*{thm_lstmmk}}.
\begin{lemma}[probability correctness] \label{lemma_probs}
  Let $w_{1:T}\in D_{m,k}$.
  For $t=1,...,t$, let $q_t = D_{m,k}(w_{1:T})$.
  Then for all $w\in\Sigma$,
  \begin{align}
    \delta(q_t, w) \not = r \leftrightarrow p_{f_\theta}(w|{w_{1:t}}) \geq \epsilon
  \end{align}
\end{lemma}
Intuitively, Lemma~\ref{lemma_probs} states that $f_\theta$ assigns greater than $\epsilon$ probability mass in context to tokens $w_{t}$ such that the prefix $w_{1:t+1}$ is the prefix of some member of $D_{m,k}$.

\paragraph{Proof of Lemma~\ref{lemma_probs}}
We proceed in cases by the state, $q$.
We'll show that a lower-bound on the probabilities of allowed symbols is greater than an upper-bound on the probabilities of disallowed symbols.
We should note, however, that this is effectively a technicality to ensure our construction is fully constructive -- that is, we provide concrete values for each parameter in the model; else, we could simply indicate that as $\pscale$ grows, the probability mass on $w$ such that $\delta(q,w)=r$ converges to $0$, while the mass on all other symbols converges 1 over the number of such symbols.

\paragraph{Case $q=[]$}
First, consider the case that $q_{t-1}=[]$.
For all $\symclose_i$, $\delta(q,\symclose_i)=r$, and for all other symbols $w$, $\delta(q,w) \not = r$.
By the stack correspondence lemmas, Lemma~\ref{appendix_lemma_lstm_stack_correspondence}~\ref{appendix_lemma_simple_rnn_stack_correspondence}, we have that $\mathcal{Q}(h_{t-1})=q_{t-1}$ (Simple RNN) or $\mathcal{Q}(c_{t-1})=q_{t-1}$ (LSTM), and $h_{t-1} = \mathbf{0}$.

For all $\symclose_i$, we have $\delta(q, \symclose_i) = r$, and we have logits:
\begin{align*}
  p_{f_\theta}&(w_t=\symclose_i|h_{t-1}) \\
  &\propto v_{\symclose_i}^\top h_{t-1} + b_{v,\symclose_i}\\
  &\leq 0 - 0.5\pscale\gamma\\
  &= - 0.5\pscale\gamma,
\end{align*}
as required.
For all $w_t=\symopen_i$, we have $\delta(q,\symopen_i) \not=r$, and we have logits:
\begin{align*}
  p_{f_\theta}&(w_t=\symopen_i|h_{t-1}) \\
  &\propto v_{\symopen_i}^\top h_{t-1} + b_{v,\symopen_i}\\
  &= 0 + 0.5\pscale\gamma\\
  &= 0.5\pscale\gamma
\end{align*}
Finally, for $\symend$, $\delta(q,\symend)\not=r$, and we have logits:
\begin{align*}
  p_{f_\theta}&(w_t=\symend|h_{t-1}) \\
  &\propto v_{\symend}^\top h_{t-1} + b_{v,\symend_i}\\
  &= 0 + 0.5\pscale\gamma\\
  &= 0.5\pscale\gamma,
\end{align*}
as required.
With the logits specified for our whole vocabulary, we can compute the partition function of the softmax function by summing over all $k$ open brackets, all $k$ close brackets, and the END bracket, to determine probabilities for each token,
\begin{align*}
  p_{f_\theta}&(w_t=\symclose_i|h_{t-1}) \\
  &=\frac{e^{-0.5\pscale}}{Z_{[]}}\\
  &Z_{[]}\leq(k+1) e^{0.5\pscale\gamma} + ke^{-0.5\pscale\gamma}
\end{align*}
while the probability for any $\symopen_i$ and $\symend$ is $p_{f_\theta}(w_t=\symopen_i|h_{t-1})=e^{0.5\pscale\gamma}/Z_{[]}$.
The partition function is computed as $e^{0.5\pscale\gamma}$ multiplied by the number of words $w$ such that $\delta(q,w)\not=r$ (in this case, $k$ open brackets and the end bracket), plus a quantity upper bounded by $e^{-0.5\pscale\gamma}$ multiplied by the number of words $w$ such that $\delta(q,w)=r$ (in this case, $k$ close brackets.)
All probabilities under this model are of this form. %

\paragraph{Case $q=[\symopen_{i_1},...,\symopen_{i_{m'}}], m'<m$}
Next, consider the case that $q = [\symopen_{i_1},...,\symopen_{i_{m'}}]$ for $0<m'<m$.
For the Simple RNN, we have by Lemma~\ref{appendix_lemma_simple_rnn_stack_correspondence} that $h_{t-1,1}=\psi^{-1}(\symopen_m')$ (the top of the stack.)
Likewise for the LSTM, we have by Lemma~\ref{appendix_lemma_simple_rnn_stack_correspondence}, that $h_{t-1,m'}=\text{tanh}(e_{i_{m'}})$, and $h_{t-1,j\not=m'} = \mathbf{0}$.
Either way, we have $v_{\symclose_i}^\top h_{t-1} \geq \pscale\gamma$ by the softmax validity property.\footnote{In particular, for the Simple RNN, this expression is equal to $\pscale$ due to the softmax validity property; because of the tanh in $h_t = o_t\odot \text{tanh}(c_t)$ in the LSTM expression, it's equal to $\pscale\gamma<\pscale$ where $\gamma=\text{tanh}(1)$.}

As with the first case, we construct the logits for each of the symbols.
First, we have $\delta(q,\symopen_{m'})\not=r$; this is the top element of the stack, which can be popped with a close bracket. We have logits:
\begin{align*}
  p_{f_\theta}&(w_t=\symclose_i|h_{t-1}) \\
  &\propto v_{\symclose_i}^\top h_{t-1} + b_{v,\symclose_i}\\
  &=  \pscale\gamma - 0.5\pscale\gamma\\
  &=  0.5\pscale\gamma
\end{align*}
as required. %
For $w_t=\symclose_i$ where $i\not=i_{m'}$, we have $\delta(q,\symopen_i)=r$ (since $\symopen_i$ is not the top of the stack, seeing close bracket $\symclose_i$ would break the well-balancing requirement), and the logits:
\begin{align*}
  p_{f_\theta}&(w_t=\symclose_i|h_{t-1}) \\
  &\propto v_{\symclose_i}^\top h_{t-1} + b_{v,\symclose_i}\\
  &\leq  0 - 0.5\pscale\gamma\\
  &=  - 0.5\pscale\gamma,
\end{align*}
again by the softmax validity property.
For $w_t=\symopen_i$, we have $\delta(q,\symopen_i)\not=r$, and the logits:
\begin{align*}
  p_{f_\theta}&(w_t=\symopen_i|h_{t-1}) \\
  &\propto v_{\symopen_i}^\top h_{t-1} + b_{v,\symopen_i}\\
  &= 0 + 0.5\pscale\gamma\\
  &= 0.5\pscale\gamma
\end{align*}
as required, since $m'<m$, and by the softmax validity property, $v_{\symopen_i}^\top h_{t,m'<m}=\mathbf{0}$.
Finally, for $\symend$, we have $\delta(q,\symend)=r$, and the logits
\begin{align*}
  p_{f_\theta}&(w_t=\symend_i|h_{t-1}) \\
  &\propto v_{\symend}^\top h_{t-1} + b_{v,\symend_i}\\
  &= -\pscale\gamma + 0.5\pscale\gamma\\
  &= -0.5\pscale\gamma,
\end{align*}
by the softmax validity property, as required.
To reason about the probabilities, we again need to construct the partition function of the softmax
\begin{align*}
  Z_p \leq (k+1)e^{0.5\pscale\gamma} + ke^{-0.5\pscale\gamma}
\end{align*}
where $Z_p$ stands for $Z$-``partial'', for a partially full stack. ($k$ open brackets are allowed, plus 1 close bracket; $k-1$ close brackets are disallowed, as well as the $\symend$ word.)

\paragraph{Case $q=[\symopen_{i_1},...,\symopen_{i_{m}}]$}
Next, consider the case that $q=[\symopen_{i_1},...,\symopen_{i_{m}}]$, that is, a full stack with $m$ elements.
For the Simple RNN, by Lemma~\ref{appendix_lemma_simple_rnn_stack_correspondence}, we have that $h_{t-1,1}=\psi^{-1}(\symopen_i)$, and $h_{t-1,m}=\psi^{-1}(\symopen_j)$ for some $j$. (The top of the stack is $\symopen_i$, and the stack is full.)
And for the LSTM, by Lemma~\ref{appendix_lemma_lstm_stack_correspondence}, we have that $h_{t-1,m} = \text{tanh}(\psi^{-1}(\symopen_{i_m}))$, and $h_{t-1,j\not=m}=\mathbf{0}$.
The logits for each symbol are as follows.
Identical to the last case, we have $\delta(q,\symclose_{i_m})\not=r$, and  $p_{f_\theta}(\symclose_{i_m}|h_{t-1}) \propto e^{0.5\pscale\gamma}$ as required; this is the element at the top of the stack.
Also identically, all other brackets have $\delta(q,\symclose_{i\not=i_m})=r$, and the values for those symbols are $\leq -0.5\pscale\gamma$, as required.
For $w_t=\symopen_i$, we have $\delta(q,\symopen_i)=r$; this is because with $m$ elements, the $m$-bound means no more elements can be pushed. The logits are:
\begin{align*}
  p_{f_\theta}&(w_t=\symopen_i|h_{t-1}) \\
  &\propto v_{\symopen_i}^\top h_{t-1} + b_{v,\symopen_i}\\
  &= -\pscale\gamma + 0.5\pscale\gamma\\
  &= -0.5\pscale\gamma,
\end{align*}
by the softmax validity property, as required.
Finally, for $\symend$, identically to the previous case, we have $\delta(q,\symend)=r$, and logits $e^{-0.5\pscale\gamma}$, as required.
The partition function is as follows:
\begin{align*}
  Z_f &\leq ke^{0.5\pscale\gamma} + (k+1)e^{-0.5\pscale\gamma}
\end{align*}
where $Z_f$ stands for $Z$-``full'', for a full stack.

\paragraph{Case $q=[\symend]$}
Last, consider the case that $q=[\symend]$.
In this case, if at timestep $t$, then symbol $w_t=\symend$, since the only transition to state $[\symend]$ is $\delta([],\symend)=[\symend]$.
Because of this and the definition of the universe of strings as $\Sigma^*\symend$, $f_\theta$ need not be defined on strings progressing from $q=[\symend]$.
Intuitively, this is because $\symend$ indicates that the string has terminated.
Hence this case vacuously holds.

We thus have three partition functions, $Z_p$, $Z_{[]}$, and $Z_f$, whose values we only have bounds for.
However, we can see that as we scale $\pscale$ large, they converge to $Z_{[]} = (k+1)e^{0.5\pscale\gamma}$, $Z_{f} = ke^{0.5\pscale\gamma}$, $Z_p=(k+1)e^{0.5\pscale\gamma}$ because the contributions from the disallowed symbols converges to zero.
Thus, the probability assigned to the lowest-probability allowed symbol converges to $\frac{1}{k+1}$ as $\pscale$ grows large, while the probability assigned to the highest-probability disallowed symbol converges to $0$.
So, we choose $\epsilon = \frac{1}{2(k+1)}$,and let $\pscale>\frac{2.4}{\gamma}$, so that $e^{0.5\pscale\gamma}> 10e^{-0.5\pscale\gamma}$.
Under this, the smallest probability assigned to any allowed symbol is lower-bounded by $\frac{e^{0.5\pscale\gamma}}{(k+1)e^{0.5\pscale\gamma} + ke^{-0.5\pscale\gamma}}>\frac{1}{(k+1)+0.1k}>\epsilon$, and the largest probability assigned to any disallowed symbol is upper-bounded by $\frac{e^{-0.5\pscale\gamma}}{ke^{+0.5\pscale\gamma}} \leq \frac{0.1}{k}=\frac{1}{10k}<\epsilon$.
This completes the proof of Lemma~\ref{lemma_probs}.

\subsection{Completing proofs of generation (Theorems~\hyperref[appendix_thm_simple_rnn_2mk]{\ref*{thm_simplernn2mk}},~\hyperref[appendix_thm_lstm_mk]{\ref*{thm_lstmmk}})}
Now that we've proved the probability correctness lemma, we're ready to complete our proof that our Simple RNN and LSTM constructions generate \dyckkm.

Recall that we've overloaded notation, calling $D_{m,k} \subset \sigma^*$ the set of strings defining the language \dyckkm.
We must show that the $\epsilon$-truncated support of $f_\theta$, which we'll call $\mathcal{L}_{f_\theta}$ is equal to $D_{m,k}$.
We show both inclusions.
\paragraph{Pf. ($D_{m,k} \subseteq \mathcal{L}_{f_\theta}$)}
Let $w\in D_{m,k}$.
We'll show $w\in \mathcal{L}_{f_\theta}$.
For all prefixes $w_{1:t}$, $t=1\dots,T$, let $q_{t-1}=D_{m,k}(w_{1:t-1})$, the state of the DFA after consuming all tokens of the prefix except the last.
For the Simple RNN, by Lemma~\ref{appendix_lemma_simple_rnn_stack_correspondence}, we have that $\mathcal{Q}(h_{t-1}) = q_{t-1}$.
For the LSTM, by Lemma~\ref{appendix_lemma_lstm_stack_correspondence}, we have that $\mathcal{Q}(s_{t-1,1},\dots,s_{t-1,m})=q_{t-1}$.
Given this, for either construction, by Lemma~\ref{lemma_probs}, we have that $p_{f_\theta}(w_t|w_{1:t-1})>\epsilon$.
Since this is true for all $t=1,\dots,T$, we have that $w$ is in the $\epsilon$-truncated support of $f_\theta$, and so $w\in\mathcal{L}_{f_\theta}$.

\paragraph{Pf. ($\mathcal{L}_{f_\theta} \subseteq D_{m,k}$)}
Let $w_{1:T}\in \mathcal{L}_{f_\theta}$.
We'll show $w_{1:T}\in D_{m,k}$ by proving the contrapositive.

Let $w_{1:T}\not\in D_{m,k}$.
We have that $w_{T}=\symend$ by definition.
From the transition function $\delta$, we know that $\delta(q,\symend)\in\{[\symend],r\}$, that is, $\symend$ transitions from any state either to the accept state or the reject state.
Because $w_{1:T}\not\in D_{m,k}$, it must be that $\delta(q_{T-1},\symend)=r$, that is, $q_T=r$.
We have no guarantee about $f_{\theta}(w_T|w_{1:T-1})$, however, since we don't know whether $w_{T-1}$ is a prefix of some string in \dyckkm. %
However, we do know that there must be some $t'$ such that the first time $q_t=r$ is for $t=t'$, that is, the first timestep in which a disallowed symbol is seen and $D_{m,k}$ transitions to the reject state $r$ (after which it self-loops in $r$ by definition.)

Consider then the prefix $w_{1:t'-1}$.
We know that $q_{t'-1}\not=r$.
So without loss of generality, let $q_{t'-1}=[\symopen_{i_1},\dots,\symopen_{i_{m'}}]$.
We then construct a string:
\begin{align*}
  w_{1:T'} = w_{1},\dots,w_{1:t'-1},\symclose_{i_1},\dots,\symclose_{i_{m'}}, \symend
\end{align*}
which simply closes all the brackets on the stack represented by $q_{t'}$.
By recursive application of $\delta$, we have that $q_{T'-1}=[]$, and thus $q_{T'}=[\symend]$, meaning $w_{1:T'}\in D_{m,k}$.

Coming back to our prefix $w_{t:t'-1}$, we now know that it is a prefix of a string $w_{1:T'}\in D_{m,k}$.
Thus, for our Simple RNN, by Lemma~\ref{appendix_lemma_simple_rnn_stack_correspondence}, we have that $\mathcal{Q}(h_{t'-1}) = q_{t'-1}$.
Likewise for our LSTM, by Lemma~\ref{appendix_lemma_lstm_stack_correspondence}, we have that $\mathcal{Q}(c_{t'-1})=q_{t'-1}$.
And by Lemma~\ref{lemma_probs}, since $\delta(q_{t'-1},w_{t'})=r$, we have $p_{f_\theta}(w_{t'}|w_{1:t'-1})<\epsilon$.
And so, $w_{1:T}\not \in \mathcal{L}_{f_\theta}$.
This completes the proof of Theorems~\hyperref[appendix_thm_simple_rnn_2mk]{~\ref{thm_simplernn2mk}},~\hyperref[appendix_thm_lstm_mk]{~\ref{thm_lstmmk}}.
\qed

\subsection{Proving the general construction in $O(k^{m+1})$}
As a corrolary of the above, we now formalize our proof that a general DFA construction in RNNs, using $O(k^{m+1})$ hidden units to simulate the DFA $D_{k,m}$ of \dyckkm, permits an RNN construction that generates \dyckkm.
Formally,
\thmnaivegeneration* \label{appendix_thm_naive_generation}

\paragraph{Proof.}
A general construction of any DFA in an RNN has $|Q||\Sigma|$ states, \cite{merrill2019sequential,giles1990higher}.
For \dyckkm, $|Q||\Sigma| \in O(k^{m+1})$.
Each state of the DFA $q \in Q$ is represented $\Sigma$ times, once for each word in the vocabulary.
If $q_t = \delta(q_{t-1}, w)$, then the hidden state is $h_{q_{t},w}$, a 1-hot vector, equal to one at an index specified by and unique to $(q_{t-1},w)$.
By defining the mapping $\mathcal{Q}(h_{q,w}) = q$, this construction obeys a stack correspondence lemma.
Finally, since the state of the RNN specifies the DFA state as a 1-hot encoding, we can define the softmax matrix $V$ as follows.
For each state $q$, we can simply set all rows of $V$ corresponding to state $q$ to explicitly encode the log-probabilities of probability distributions we just proved in (\S~\ref{appendix_sec_generating}).
\qed

\section{Extending to generation in $O(m\log k)$ hidden units} \label{appendix_sec_efficient_generating}

In this section, we prove an $O(m \log k)$ upper bound on the number of hidden units necessary to capture \dyckkm with an RNN, matching the $\Omega(m\log k)$ lower-bound. 
This is accomplished by defining a new mapping $\psi$ from slots $s_{t}$ to open brackets $\symopen_i$ such that we can encode $k$ open brackets using just $3\log k-1$ hidden units while maintaining the stack correspondence lemmas and \dyckkm generation properties.
Formally, for the Simple RNN:
\thmsimplernnmlogk* \label{appendix_thm_simple_rnn_mlogk}
Next, for the LSTM:
\thmlstmmlogk* \label{appendix_thm_lstm_mlogk}

\subsection{$O(\log k)$ vocabulary encoding in the Simple RNN.}\label{sec:logk-srnn}
Intuitively, a simple way to encode $k$ elements in $\log k$ space without making use of floating-point precision is to assign each element one of the $2^{\log k}$ binary configurations of $\{0,1\}^{\log k}$.
Our construction will build off of this.

Recall that there are $k$ open brackets that need encodings in $O(\log k)$ space.
Let $p^{(i)}$, for $i\in\{1,\dots,k\}$, be the $i^{\text{th}}$ member of an arbitrary ordering of the set $\{0,1\}^{\log k}$, that is, the set of $\log k$ binary variables.
Let $\mathbf{1}\in\mathbb{R}^{\log k-1}$ be the vector of all $1$s.

Then the encoding of symbol $i$ is:
\begin{align}
  \psi^{-1}(\symopen_i)_* = \pscale[p^{(i)}; (1-p^{(i)}); \mathbf{1}] \in \{0,1\}^{3\lceil\log k\rceil -1}
\end{align}
Where the semicolon $(;)$ denotes concatenation.
This can be efficiently implemented in our simple RNN construction by modifying the $U$ matrix to encode each $\psi^{-1}$.
Each row of the softmax matrix $V$ is as follows:
\begin{align}
  &v_{\symclose_i,1} = \pscale[p^{(i)}; (1-p^{(i)}); -\mathbf{1};\mathbf{0}]\\
  &v_{\symopen_i,m} = -\pscale[\mathbf{1}\in\mathbb{R}^{2\lceil\log k\rceil}; -\mathbf{1}\in\mathbb{R}^{\lceil\log k\rceil - 1}]\\
  &v_{\symend,j} = -\pscale[\mathbf{1}\in\mathbb{R}^{2\lceil \log k\rceil}; -\mathbf{1}\in\mathbb{R}^{\lceil\log k\rceil - 1}]
\end{align}
where $j$ is for all $j\in[1,\dots,m]$, and all slots not specified are equal to $\mathbf{0}$.

Now it suffices to prove the stack correspondence lemma and probability correctness lemmas using $\psi_*$ and $V$.

\paragraph{Proof of stack correspondence lemma, Lemma~\ref{appendix_lemma_simple_rnn_stack_correspondence}.}
As noted in $(\S~\ref{appendix_sec_simple_rnn})$, the only relevant property of the encodings $e_i$ used for symbols $\symopen_i$ in Lemma~\ref{appendix_lemma_simple_rnn_stack_correspondence} is that it take on values in $\{0,1\}$.
This is also true of $\psi^{-1}_*$, so the lemma still holds.

\paragraph{Proof of probability correctness lemma, Lemma~\ref{lemma_probs}.}
It suffices to show that $V$ obeys the softmax validity property (Definition~\ref{appendix_def_softmax_validity}) with respect to $\psi_*$ to prove Lemma~\ref{lemma_probs} using new encoding $\psi_*$.
First, we have that 
\begin{align*}
  v_{\symclose_j,1}^\top \psi_*^{-1}(\symopen_i) &= \Big(\sum_{i=1}^{\lceil \log k \rceil}p^{(i)}p^{(j)} \\
  &+ \sum_{i=1}^{\lceil \log k \rceil}(1-p^{(i)})(1-p^{(j)}) \\
  &+ \sum_{i=1}^{\lceil \log\rceil k -1} 1(-1)\Big)\\
  &\begin{cases}
    = \log k - (\log k -1) & i=j\\
    \leq (\log k -1) - (\log k -1) & i\not=j\\
  \end{cases}\\
  &\begin{cases}
    = 1 & i=j\\
    \leq 0 & i\not=j\\
  \end{cases},
\end{align*}
Where $v_{\symclose_j,1}$ is specified since the top of the stack is always in slot $1$, ensuring the first requirement.
Intuitively, this dot product simply counts up the number of bits that agree between the row $j$ and the symbol $i$ encoded; if they're the same, all bits agree; if not, at least 1 must disagree.

Further, we have
\begin{align*}
  v_{\symopen_i,m}^\top \psi_*^{-1}(\symopen_i) &= -\pscale\Big(\sum_{\ell=1}^{\lceil \log k \rceil}1(p^{(i)}_\ell)\\
    &+ \sum_{\ell=1}^{\lceil \log k \rceil} 1(1-p^{(i)}_\ell) \\
&+ \sum_{\ell=1}^{\lceil \log k \rceil-1}(-1*1)\Big)\\
&= -\pscale\big((\lceil\log k\rceil) - (\lceil\log k \rceil - 1)\big)\\
&= -\pscale 
\end{align*}
since the count of bits $p^{(i)}_\ell$ that are $1$ plus the count of negated bits $(1-p^{(i)}_\ell)$ that are one must always equal exactly $\lceil \log k\rceil$.
This holds identically for $v_{\symend,j}^\top \psi_*^{-1}(\symopen_i)$, for all $j\in[m]$.
Finally, $v_{\symopen_i,m'} =0$, for $m'<m$, ensuring that $v_{\symopen_i,m'<m}^\top\psi^{-1}_*(\symopen_i)=0$, as required.
This proves that $\psi_*$ obeys the softmax validity properties, so Lemma~\ref{lemma_probs} still holds for our Simple RNN construction.

\paragraph{Counting hidden units}
Since $\psi_*$ encodes each symbol in $d=3\lceil \log k \rceil -1$ space, and the Simple RNN construction constructs a stack in $2md$ space, we have that Simple RNNs can generate \dyckkm in $2m(3\lceil \log k \rceil - 1 ) = 6m\lceil\log k \rceil - 2m$ space.
This proves Theorem~\hyperref[appendix_thm_simple_rnn_mlogk]{\ref*{thm_simplernnmlogk}}.

\subsection{$O(\log k)$ vocabulary encoding in the LSTM}
We'll use a slight variation of the construction we used for the Simple RNN, made possible since the LSTM can encode the value $-1$ due to the hyperbolic tangent.
First, for the encoding, we swap the last $\lceil \log k \rceil -1$ values $1$ for $-1$:
\begin{align}
  &\psi'^{-1}(\symopen_i) = [p^{(i)}; (1-p^{(i)}); -\mathbf{1}]\\
\end{align}

And negate the corresponding values of the softmax matrix, as follows:
\begin{align}
  &v_{\symopen_i,m} = -\pscale[\mathbf{1}\in\mathbb{R}^{2\lceil\log k\rceil}; -\mathbf{1}\in\mathbb{R}^{\lceil\log k\rceil - 1}]\\
  &v_{\symend,j} = -\pscale[\mathbf{1}\in\mathbb{R}^{2\lceil \log k\rceil}; -\mathbf{1}\in\mathbb{R}^{\lceil\log k\rceil - 1}]
\end{align}
where $j$ is for all $j\in[1,\dots,m]$, and all slots not specified are equal to $\mathbf{0}$.

Now it suffices to prove the stack correspondence lemma and probability correctness lemmas using $\psi_*$ and $V$.

\paragraph{Proof of stack correspondence lemma, Lemma~\ref{appendix_lemma_lstm_stack_correspondence}.}
It suffices to show that $\psi_*$ obeys the encoding-validity property (Definition~\ref{appendix_def_encoding_validity}).
First, for all $i$, we have,
\begin{align*}
  \mathbf{1}^\top \psi_*^{-1}(\symopen_i) &= \mathbf{1}^\top[p^{(i)}; (1-p^{(i)})] - (\log k- 1)\\
  &= \log k - (\log k - 1) = 1,
\end{align*}
as required.
This was possible because we could use the $-1$ value in the encoding $\psi_*$.
Second, we see that $\psi_*^{-1}(\symopen_i) \in \{0,1\}^{3\lceil \log k \rceil -1}$, as required.
Third, we have still that $\psi_*(\mathbf{0})=\mathbf{0}$.
Fourth and finally, we have by construction that all encodings $\psi_*^{-1}(\symopen_i)$ are distinct and none are equal to zero, by construction, since each is assigned a different bit configuration (and each bit configuration is concatenated to its negation.)

So, the encoding-validity property holds on $\psi_*$, and the stack correspondence lemma, Lemma~\ref{appendix_lemma_lstm_stack_correspondence} holds.

\paragraph{Proof of probability correctness lemma, Lemma~\ref{lemma_probs}.}
The proof of the probability correctness lemma holds as an immediate corollary of the proof of the same lemma for the Simple RNN using $\psi_*$.
In particular, the only difference between the $\psi_*$ and $V$ used for the Simple RNN and that used for the LSTM is the swapping of a factor of $-1$ from a span of $v$ to that of $\psi_*$, so we still have
\begin{align}
  v_{\symclose_j,u}^\top \psi_*^{-1}(\symopen_i)  = 
  &\begin{cases}
    = 1 & i=j\\
    \leq 0 & i\not=j\\
  \end{cases},
\end{align}
where $u$ picks out the top of the stack from the stack correspondence lemma.
We identically have 
\begin{align}
  v_{\symopen_i,m}^\top \psi_*^{-1}(\symopen_i) &= -\pscale
\end{align}
again from the proof for the Simple RNN, and likewise for $v_{\symend,j}^\top \psi^{-1}(\symopen_i)$ for all $j\in[m]$.
Finally, we have $v_{\symopen_i,m'}=\mathbf{0}$, for $m'<m$, again from the proof of the simple RNN.
This proves that $\psi_*$ obeys the softmax validity properties, so Lemma~\ref{lemma_probs} still holds for our LSTM construction.

\paragraph{Counting hidden units}
Since $\psi_*$ encodes each symbol in $d=3\lceil \log k \rceil -1$ space, and the LSTM construction constructs a stack in $md$ space, we have that LSTM can generate \dyckkm in $m(3\lceil \log k \rceil - 1 ) = 3m\lceil\log k \rceil - m$ space.
This proves Theorem~\hyperref[appendix_thm_lstm_mlogk]{\ref*{thm_lstmmlogk}}.

\subsection{Accounting of finite precision.}
\label{appendix_accounting_of_finite_precision}
We've used a number of constants in our proofs; we now enumerate them to make it clear that our construction can operate in a finite-precision setting.
We use the symbols as defined in their respective sections.
The only values written to vector memory are $\{0,1\}$.
In the recurrent equations, we use $\{\beta,-\beta, 2\beta, -2\beta, -3\beta\}$. We refer to $\text{tanh}(1)$ as $\gamma$, and use $\{\gamma, -\gamma, -2\gamma\}$.
We use $\rscale$ as a constant as well in our recurrent parameters, and use the constants $\{-\gamma\rscale,-2\gamma\rscale, \frac{1}{2}\gamma\rscale, -\frac{1}{2}\gamma\rscale, \frac{3}{2}\gamma\rscale, -\frac{3}{2}\gamma\rscale\}$.
None of these constants grow or shrink with $m$ or $k$.
However, when defining a distribution that must assign mass to, e.g., all $k$ open brackets, some values \textit{must} scale with $k$.
This is okay, since the number of \textit{unique} values need not scale with $k$; some of the values in finite set we use just become very small as $k$ grows.
We use $\pscale$ in our definitions of $V$ and $b_v$, as well as the constants $\{-\pscale, \frac{1}{2}\pscale, -\frac{1}{2}\pscale,\pscale\gamma, -\pscale\gamma, \frac{1}{2}\pscale\gamma,-\frac{1}{2}\pscale\gamma\}$ in the equations in that section.
Likewise our threshold $\epsilon$ for $\epsilon$-truncated support shrinks with $k$, as do the probabilities assigned by the model to tokens. For those probabilities, we choose to represent $\epsilon$; thus allowed symbols (with probability $>\epsilon$) will be evaluated correctly as $\geq\epsilon$. For disallowed symbols, we choose to represent the value half-way between $\epsilon$ and our upper-bound on disallowed symbols' probability, $\frac{1}{10k}$, so these symbols' probabilities will be correctly evaluated as $<\epsilon$. Finally,  we refer to $\log k$ in our intermediate equations as well (though it need not be computed explicitly.).
Our finite-precision set $\mathbb{P}$ is the union of all finitely many constants named in this section.

\section{Experiment Details} \label{appendix_sec_experiments}
In this section, we provide detail on our preliminary study on LSTM LMs learning \dyckkm from samples.
\subsection{Data}
We run experiments on \dyckkm for $k\in\{2,8,32,128\}$ and $m\in\{3,5\}$. 

\paragraph{Distribution over \dyckkm.}
As \dyckkm is an (infinite) set, in order to train language models on it, we must define a probability distribution over it.
Intuitively, we sample tokens conditioned on the current DFA (stack) state.
Depending on the stack state, one or two of the actions $\{$\texttt{push} $\symopen_i$, \texttt{pop}, \texttt{end}$\}$ are possible.
For example, in the empty stack state, \texttt{push} $\symopen_i$ and \texttt{end} are possible.
We sample uniformly at random from the possible actions.
Conditioned on choosing the \texttt{push} $\symopen_i$ action, we sample uniformly at random from the open brackets $\symopen_i$.
Conditioned on choosing the \texttt{pop} action, we generate  with probability 1 the $\symclose_i$ that corresponds to the $\symopen_i$ on the top of the DFA stack (to obey the well-balancing condition.
Upon choosing the \texttt{end} action, we generate $\omega$ and terminate.

Formally,
\begin{align}
&q_0 = []\\
&a_t \sim \begin{cases}
  U(\{\text{\texttt{push} $\symopen_i$, \texttt{end}}\}) & |q| = 0\\
  U(\{\text{\texttt{push} $\symopen_i$, \texttt{pop}}\}) & 0 < |q| < m\\
  U(\{\text{\texttt{pop}}\}) & |q| = m
\end{cases}\\
&w_t \sim \begin{cases}
U(\{\symopen_i\}_{i\in 1:k}) & a_t = \text{\texttt{push} } \symopen_i\\
U(\{\symclose_j\}) & a_t = \text{\texttt{pop}}, q = [\dots,\symopen_j]\\
U(\{\symend\}) & a_t = \text{\texttt{end}}\\
\end{cases}
\end{align}

We choose this distribution to provide the following statistical properties.
Consider the markov chain with nodes $1$ to $m$ representing $|q_t|$, the number of symbols on the DFA state stack at timestep $t$.
(This is a Markov chain because $q_t$ is sufficient to describe the probability distribution over all suffixes after timestep $t$.)
From the probability distribution we've defined, for any timestep where the stack is neither empty nor full, that is, $0 < |q_t| < m$ there is probability $1/$ of advancing in the markov chain towards $m$, that is $|q_{t+1}| = |q_t|+1$, and probability $1/2$ of retreating towards $0$, that is, $|q_{t+1}| = |q_t|-1$.
This \textit{hitting time} of state $m$ from state $0$ is the expected number of timesteps it takes for the generation sequence to start at the empty stack state $q_{t}=[]$ and end up at a full stack state $|q_{t'}|=m$.
Because of the $1/2$ probability of advancing or retreating along the markov chain, the hitting time is $O(m^2)$.

\begin{table}
\centering
\begin{tabular}{l c c}
\toprule
$m$ & 3 & 5\\
\midrule
Train lengths & 1:84 & 1:180\\
Test lengths & 85:168 & 181:360\\
\bottomrule
\end{tabular}
\caption{\label{appendix_table_lengths}Training and testing length cutoffs (min/max) for the distributions we define over \dyckkm}
\end{table}

\paragraph{Length conditions.}
\citet{suzgun2018evaluating} showed that the choice of formal language training lengths has a significant effect on the generalization properties of trained LSTMs.
We choose training lengths carefully, keeping in mind the $O(m^2)$ hitting time from empty to full stack states, to empirically ensure that in expectation, the longest training sequences traverse from the empty state to a full stack state $|q_t|=m$ and back to the empty stack state at least three times.
This ensures that models are not able to use simple heuristics at training time, like remembering the first open bracket in the sequence to close the last close bracket.
The exact length statistics are provided in Table~\ref{appendix_table_lengths}.
Length statistics for training and testing sets are shown in Figure~\ref{appendix_figure_lengths}..

\begin{figure}
\includegraphics[width=\linewidth]{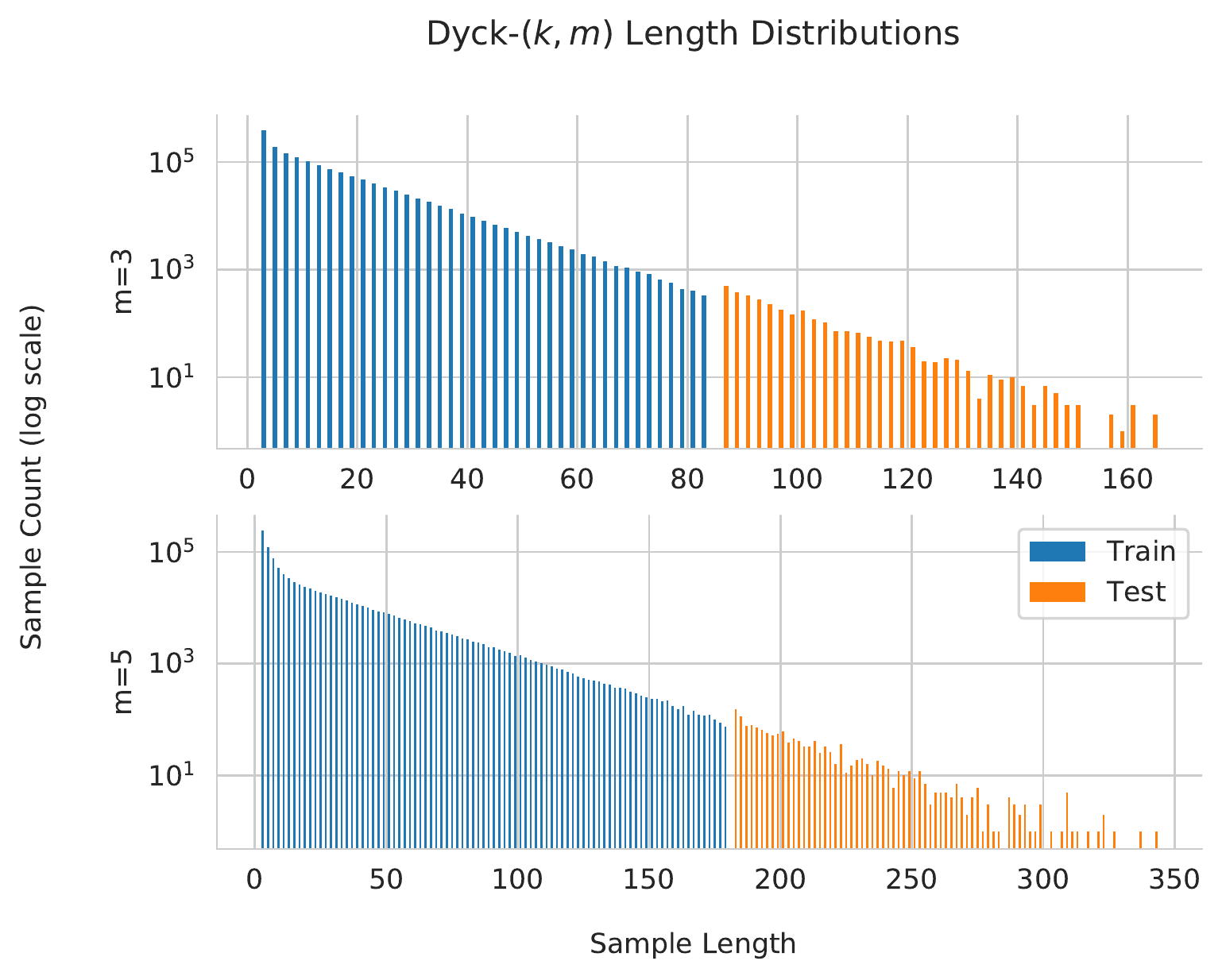}
\caption{\label{appendix_figure_lengths} Sample length statistics for the twenty million token datasets. Note that the y-axis is on a log scale.}
\end{figure}

\paragraph{DFA state analysis.}
\begin{table}
\centering
\begin{tabular}{c r r r r}
\toprule
$k$ & 2 & 8 & 32 & 128\\
\midrule
\% of Total & 100 & 100 & 15 & .03\\
\% of Test & 100 & 100 & 58 & 31\\
\bottomrule
\end{tabular}
\caption{\label{appendix_table_dfa_percents}Percent of DFA states seen in twenty million tokens of training data, as measured with respect to both the set of all DFA states and the set of DFA states seen at test time.}
\end{table}

Since \dyckkm is a regular language, it is reasonable to believe it may learn equivalences between strings that result in the same DFA state, but fail to generalize to DFA states not seen during training time.
In Table~\ref{appendix_table_dfa_percents}, we see that for $k$ equal to $2$ an $8$, all DFA states are seen at training time.
Equivalently, every possible stack configuration of $2$ or $8$ brackets, of stack sizes up to $3$ or $5$, are seen during training time.
For $k=32$, however, only $15\%$ of all possible DFA states are seen at training time, and only $58\%$ of DFA states seen at testing time are also seen at training time.
For $k=128$, the numbers are even more stark, where $0.3\%$ of all possible states, and $31\%$ of states seen at testing time are also seen at training time.
Thus, the ability of models to generalize to the test set for $k$ equal to $32$ and $128$ shows that the learned LSTMs are not simply memorizing DFA states from training time.\footnote{There are over 34 billion possible DFA states for $k=128,m=5$.}
Instead, we speculate that they're performing stack-like operations, and leave further investigation to future work.

\paragraph{Sample counts}
To test sample efficiency of learning, we study four dataset sizes: $\{2k, 20k, 200k, 2m, 20m\}$ tokens for training for each $k,m$ combination.
In all training settings, we use identical development and test sets of size $20k$ and $300k$ tokens, respectively.
The development set is sampled from the training distribution.

\subsection{Models}
We use LSTMs, defined as in the main text with a linear readout layer, and implemented in PyTorch \cite{paszke2019pytorch}.
We set the hidden dimensionality of the LSTM to $3m\log(k)-m$, and the input dimensionality to $2k+10$.

\subsection{Training}
We use the default LSTM initialization provided by PyTorch.
We train using Adam \cite{kingma2014adam}, using a starting learning rate of $0.01$ for all training sets less than $2m$ tokens.
Based on hand hyperparameter optimization, we found that for $k=128$, at $2m$ tokens it was better to use a starting learning rate of $0.001$.
For training sets of size $20m$ tokens, we use a starting learning rate of $0.001$ for all settings of $k$.
We use a batch size of $10$ for all experiments.
We evaluate perplexity on the development set after every epoch, restarting Adam with a learning rate decayed by $0.5$ if the development perplexity does not achieve a new minimum.
After three consecutive epochs without a new minimum development perplexity, we stop training.
We use no explicit regularization.

\subsection{Evaluation}
We're interested in evaluating the behavior of the LSTM LMs in hierarchical memory management, that is, in remembering what type of bracket $\symclose_i$ can come next.
The model cannot possibly do better than random at predicting the next open bracket.
Other aspects of the language, like whether the string can end, or if the stack is full, can be solved easily with a counter, which LSTMs are known to implement \cite{weiss2018practical,suzgun2019lstm}.
We thus evaluate whether, for each observed close bracket, the LM is confident about which close bracket (that is, $i$ for $\symclose_i$) it is, since this is deterministic; it must be the close bracket corresponding to the open bracket at the top of the stack of $q_{t-1}$.
We do this by normalizing the probability assigned by the LM to the \textit{correct} close bracket by the sum of probabilities assigned to \textit{any} close bracket:
\begin{align}
p(\symclose_j|\symclose) = \frac{p(\symclose_j)}{\sum_{i}p(\symclose_i)}
\end{align}
We evaluate whether the model is confident, when we define as $p(\symclose_j|\symclose)>0.8$.

Most pairs of open and close brackets in our distribution of \dyckkm are near linearly, but we're interested in \textit{long-distance} memory management.
Thus, we let $p_\ell$ be the empirical probability that the model confidently predicts a close bracket, conditioned on it being separated from its open bracket by $\ell$ tokens. 
We then evaluate models by taking $\text{mean}_{\ell} p_\ell$; that is, to evaluate its performance at closing brackets?
To get a value of $1$ for this metric, models must confidently close every bracket.

We train with three independent seeds for each training setting (training set size, $k$, and $m$), and report the median across seeds of our bracket closing metric in Figure~\ref{fig_mainpaper_learningdyck}.